\documentclass{article}
\usepackage{iclr2026_conference,times}
\iclrfinalcopy

\usepackage{hyperref}
\usepackage{url}

\author{Yongchang Hao$^\diamondsuit$ \quad Lili Mou$^{\diamondsuit\spadesuit}$ \\
$^\diamondsuit$Dept. Computing Science \& Alberta Machine Intelligence Institute (Amii), University of Alberta \\
$^\spadesuit$Canada CIFAR AI Chair \\
\texttt{yongcha1@ualberta.ca} \quad \texttt{doublepower.mou@gmail.com}
}

\usepackage{algorithm}
\usepackage{algorithmic}
\usepackage{amsthm}
\usepackage{amsmath}
\usepackage{amsfonts}
\usepackage{amssymb}
\usepackage{physics}
\usepackage{graphicx}
\usepackage{xspace}
\usepackage{multirow}
\usepackage{subcaption}
\usepackage{thmtools}
\usepackage{thm-restate}
\usepackage{booktabs}

\DeclareRobustCommand{\logo}{%
  \begingroup\normalfont
  \raisebox{-0.25em}{%
  \hspace{-0.5em}
  \includegraphics[height=1.5em]{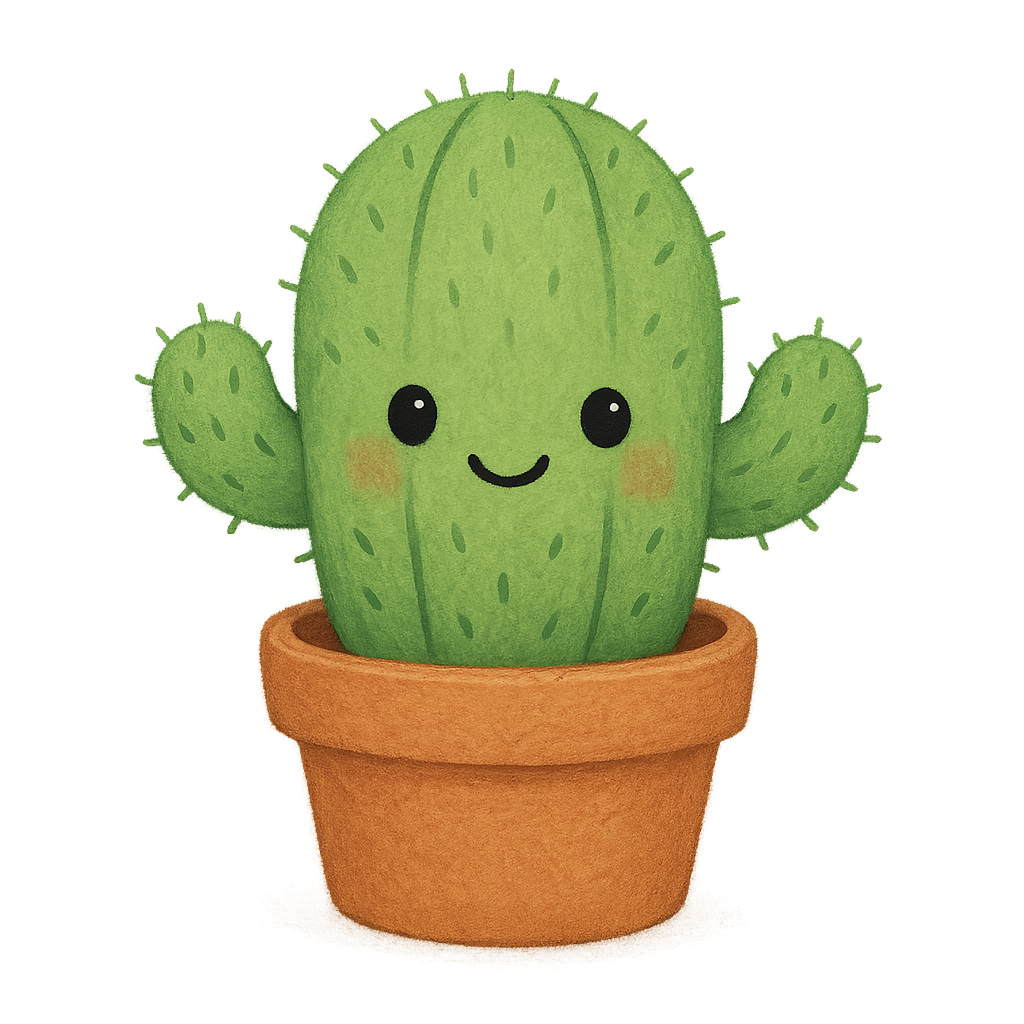}%
  }%
  \kern 0.1em
  \endgroup
}

\newcommand{\name}{Cactus\xspace}
\title{\logo\name: Accelerating Auto-Regressive Decoding with Constrained Acceptance Speculative Sampling}

\usepackage{amsmath,amsfonts,amsthm,bm,physics}

\def\1{\bm{1}}

\def\ve{{\bm{e}}}

\def\vg{{\bm{g}}}
\def\vh{{\bm{h}}}

\def\vr{{\bm{r}}}

\def\vx{{\bm{x}}}

\DeclareMathAlphabet{\mathsfit}{\encodingdefault}{\sfdefault}{m}{sl}
\SetMathAlphabet{\mathsfit}{bold}{\encodingdefault}{\sfdefault}{bx}{n}

\def\sN{{\mathbb{N}}}

\def\sR{{\mathbb{R}}}

\newcommand{\E}{\mathop{\mathbb{E}}}

\DeclareMathOperator*{\argmax}{arg\,max}

\theoremstyle{plain}
\newtheorem{theorem}{Theorem}

\newtheorem{lemma}[theorem]{Lemma}

\theoremstyle{definition}

\theoremstyle{remark}

\begin{document}
\maketitle

\begin{abstract}

Speculative sampling (SpS) has been successful in accelerating the decoding throughput of auto-regressive large language models by leveraging smaller draft models. SpS strictly enforces the generated distribution to match that of the verifier LLM. This is unnecessarily restrictive as slight variations of the verifier's distribution, such as sampling with top-$k$ or temperature, would also be acceptable. Typical acceptance sampling (TAS) alleviates this issue by accepting more tokens using entropy-based heuristics. However, this approach distorts the verifier distribution, potentially degrading output quality when the verifier encodes critical information.
In this work, we formalize the speculative sampling algorithm through the lens of constrained optimization. Based on this formulation, we propose \textbf{Cactus} (\textbf{c}onstrained \textbf{ac}cep\textbf{t}ance spec\textbf{u}lative \textbf{s}ampling), a method that guarantees controlled divergence from the verifier distribution and increasing acceptance rates. Empirical results across a wide range of benchmarks confirm the effectiveness of our approach. The code is publicly available.\footnote{\url{https://github.com/MANGA-UOFA/Cactus}}

\end{abstract}

\section{Introduction}\label{sec:intro}

Auto-regressive large language models (LLMs) have driven remarkable advances in machine learning and artificial intelligence~\citep{vaswani2017attention,brown2020language,kaplan2020scaling}, yet their growing size comes with steep computational costs: generating each token requires a memory-bound forward pass through hundreds of billions of parameters, which bottlenecks LLM throughput~\citep{yuan2024llm}. Speculative sampling (SpS) addresses this by first using a smaller draft model to propose a certain number of candidate tokens autoregressively, then verifying the candidate tokens in parallel with the large-scale \emph{verifier} LLM~\citep{stern2018blockwise, xia2022speculative,leviathan2023fast,chen2023accelerating}. Since SpS can emit multiple tokens per large-model invocation, it substantially speeds up auto-regressive generation by alleviating the memory-bound issue.

Despite its success, SpS enforces strict distributional equivalence with the verifier, causing correct but lower-probability tokens to be rejected. In real-world applications, exact adherence to the original distribution is generally not required~\citep{holtzman2019curious,meister2020if}. Typical acceptance sampling (TAS;~\citeauthor{cai2024medusa},~\citeyear{cai2024medusa}) mitigates this issue by accepting proposals based on entropy-driven heuristics~\citep{hewitt2022truncation, meister2023locally}. However, we show in this paper that TAS improves acceptance rates at the cost of distorting the verifier's output distribution and risking semantic drift when the verifier encodes critical information.

In this work, we reformulate speculative sampling as a constrained optimization problem, explicitly trading off acceptance rate against divergence from the verifier's distribution. Guided by this theory, we introduce Cactus (\textbf{c}onstrained \textbf{ac}cep\textbf{t}ance spec\textbf{u}lative \textbf{s}ampling), a simple yet principled modification that enforces a hard constraint on distributional divergence while enabling higher acceptance rates.

We conducted experiments on a wide range of benchmarks with multiple state-of-the-art large language models. Results show that Cactus consistently improves generation throughput compared with the lossless SpS. In addition, Cactus preserves the generation quality and diversity of the verifier model, due to its explicit divergence constraint.

\section{Approach}\label{sec:approach}

We first provide a generalized formulation for speculative sampling. This enables a theoretical analysis of speculative sampling under a constrained optimization framework. Based on this analysis, we propose a new algorithm, \name, which provably approximates the verifier distribution $q$ while achieving higher acceptance rates.

\subsection{Generalization of speculative sampling}\label{sec:approach:generalization}

\paragraph{Speculative sampling.} The vanilla speculative sampling (SpS;~\citeauthor{chen2023accelerating},~\citeyear{chen2023accelerating}) uses a \emph{draft model} $p(\cdot | \vx_{<t})$ that has a significantly smaller memory footprint than the \emph{verifier model} $q(\cdot | \vx_{<t})$. At a time step $t$, SpS repeatedly samples $m \in \sN_{+}$ tokens $x_{t}, \dots, x_{t+m-1}$ from $p$ in an auto-regressive manner. Then the verifier evaluates the probabilities of the tokens and calculates the \emph{acceptance rate} $\phi (x_{t+i} | \vx_{<t+i}) = \min\{1, q(x_{t+i} | \vx_{<t+i})/p(x_{t+i} | \vx_{<t+i})\}$ for all $i \in [0, m)$. If any token $x_{t+j}$ is rejected, then tokens $x_{t+j+1}, \dots, x_{t+m-1}$ are also discarded. As a backup, SpS resamples $x_{t+j}$ using the \emph{recover distribution} $g(x_{t+j} | \vx_{<t+j}) \propto (q(\cdot|\vx_{<t+j}) - p(\cdot | \vx_{<t+j}))_{+}$, where $(\cdot)_+$ denotes $\max(0, \cdot)$. The final accepted tokens are $x_t, \dots, x_{t+j}$. By this draft-and-verify scheme, SpS accelerates auto-regressive decoding by avoiding the need to load the large verifier model $q$ from memory at every step. This approach has been shown effective in practice~\citep{zhou2024survey,hu2025speculative}.

We formalize the draft-and-verify scheme as Algorithm~\ref{alg:generalization}. Under this setting, we can show that the vanilla SpS algorithm~\citep{chen2023accelerating} produces any target distribution with an optimal acceptance rate.

\begin{restatable}{observation}{generalization}\label{thm:generalization}
Consider any desired target distribution $h$ and draft model $p$. Algorithm~\ref{alg:generalization} produces the target distribution $h$ exactly if the acceptance rate and recover distribution are defined as
\begin{align}
    & \phi(x_t | \vx_{<t}) = \min\left\{\frac{h(x_t | \vx_{<t})}{p(x_t | \vx_{<t})}, 1\right\} \\
    \text{and} \quad & g(x_t | \vx_{<t}) = \frac{h(x_t | \vx_{<t}) - p(x_t | \vx_{<t}) \phi(x_t | \vx_{<t})}{\E_{x' \sim p}[1 - \phi(x' | \vx_{<t})]},
\end{align}
respectively. In addition, this acceptance rate $\phi$ is optimal for achieving the highest acceptance rate.
\end{restatable}

\begin{proof}
See Appendix~\ref{apx:prf:generalization}.
\end{proof}

\begin{algorithm}[t]
    \caption{Generalized formulation of speculative sampling}
       \label{alg:generalization}
    \begin{algorithmic}[1]
       \REQUIRE sampling steps $m$, draft model $p$, acceptance rate $\phi$, and recover distribution $g$
        \STATE{$t \gets 1, \vx_{<t} \gets {[\text{BOS}]}$}
       \WHILE{not end}
            \STATE{}\COMMENT{Drafting $m$ tokens}
            \FOR{$i \gets 0, \dots, m-1$}
                \STATE{$x_{t+i} \sim p(\cdot|\vx_{<t+i})$} \hfill \COMMENT{$\vx_{<t+i}$ is concatenation of $\vx_{<t}$ and $[x_t, \dots, x_{t+i-1}]$}
                \STATE{$u_i \sim U(0, 1)$} \hfill \COMMENT{$U(0,1)$ is the uniform distribution between $[0, 1]$}
            \ENDFOR

            \STATE{$c \gets \min \{ j: u_j >\phi(x_{t+j}|\vx_{<t+j}) \} \bigcup \{m\}$} \hfill \COMMENT{$c$ is the length of accepted draft tokens}
            \STATE{$x_{t+c} \sim g(\cdot | \vx_{<t+c})$} \hfill \COMMENT{$x_{t+c}$ is sampled from the recover distribution}
            \STATE{$t \gets t + c + 1$}
        \ENDWHILE    
    \end{algorithmic}
\end{algorithm}

\subsection{Approximating SpS as constrained optimization}\label{sec:approach:approx}
Observation~\ref{thm:generalization} provides a foundation to produce an arbitrary target distribution with the optimal design under the draft-and-verify scheme. Our insight is that, instead of producing a fixed verifier distribution $q$ as the target distribution $h$ like SpS, we may utilize this observation to dynamically select a distribution $h$ close to $q$ while yielding higher acceptance rates based on function $\phi$. This can be formulated as a constrained optimization problem.

For each step $t$, assume the drafted token has index $n$. Let $\vh \in \sR^{|V|}$ be the parameters to be optimized. The ideal $h$ is given by $h(i | \vx_{<t}) = \vh_i^*$, where $\vh^*$ is the solution of the following problem:
\begin{align}
        \max_{\vh} \ \ & \min \{ h_n / p(n| \vx_{<t}), 1 \} \label{def:instance-prob} \\
        \mathrm{s.t.\ } & \vh \in \Delta^{|V|-1} \\
        &  D_f(\vh \| q(\cdot | \vx_{<t})) \le \delta.
\end{align}
Here, the hyper-parameter $\delta \ge 0$ controls the closeness to the verifier model $q$, and $D_f$ is any $f$-divergence metric used to measure the distance between $q$ and $h$. Once the optimal $\vh$ is found, we can then derive the corresponding $\phi$ and $g$ by Observation~\ref{thm:generalization}.

The above formulation falls into the framework of constrained convex optimization, which we show has the following solution.

\begin{restatable}{theorem}{generalsolution}\label{thm:general-solution}
The optimal solution of $\vh$ in objective~\eqref{def:instance-prob} is 
\begin{align}
  h_i = \begin{cases}
      \gamma^*, & \text{if } i=n, \\
      \frac{1-\gamma^*}{1-q(n | \vx_{<t})} q(i|\vx_{<t}), & \text{otherwise},
  \end{cases}\label{eq:hsol}
\end{align}
where $\gamma^*$ is any root of the equation
\begin{align}
    \delta = q(n | \vx_{<t}) f\left(\frac{\gamma}{q(n | \vx_{<t})}\right) + (1 - q(n | \vx_{<t})) f\left(\frac{1 - \gamma}{1-q(n | \vx_{<t})}\right)
\end{align}
over the interval $[q(n|\vx_{<t}), +\infty)$, clamped into $[q(n|\vx_{<t}), 1]$. The function $f$ is the one used in the definition of $f$-divergence.
\end{restatable}
\begin{proof}
See Appendix~\ref{apx:prf:general-solution}.
\end{proof}

Theorem~\ref{thm:general-solution} theoretically characterizes the trade-off between closeness to the verifier model $q$ and the acceptance rate induced by $\phi$. In particular, the theorem suggests that the drafted token now has at least the same or a higher chance of being accepted (since $\gamma^* \ge q_n$). For other non-sampled tokens, their probabilities are scaled down proportionally so that $h$ remains a valid distribution.

It is worth noting that, since the solved $\vh$ in Equation~\eqref{eq:hsol} depends on the sampled token $n$, the solution is  different for different sampled tokens. As a result, the effective distribution of the overall algorithm $\vh_\text{alg}$ might have a divergence other than $\delta$ from the target distribution $q$. To this end, we provide the following theorem to guarantee the controlled divergence of the effective distribution.

\begin{restatable}{theorem}{hbound}\label{thm:hbound}
   Let $\phi_n$ and $g_n$ denote the functions that follow the solution in Theorem~\ref{thm:general-solution} when the sampled token is $n$. The distribution of the overall algorithm is given by \begin{align}
       \vh_\text{alg} = \sum_{n\in [|V|]}p(n | \vx_{<t}) \left[\phi_n(n) \ve_n + (1-\phi_n(n)) \vg_n \right],
   \end{align}
   where $\ve_n$ is a one-hot vector with only non-zero element at index $n$.
   In addition, \begin{align}
       D_f(\vh_\text{alg} \| q(\cdot | \vx_{<t})) \le \min\{\Gamma(\delta), D_f(p(\cdot | \vx_{<t}) \| q(\cdot | \vx_{<t})) \}
   \end{align}
   for any $\delta \ge 0$. Here, the function $\Gamma: [0, +\infty) \to [0, +\infty]$ is continuous and non-decreasing in $\delta$ with a value of $0$ at $\delta=0$.
\end{restatable}
\begin{proof}
    See Appendix~\ref{apx:prf:hbound}.
\end{proof}

In essence, despite the $\vh$ in Equation~\eqref{eq:hsol} is solved specifically for the sampled token $n$, the divergence between the overall distribution and the target distribution is still implicitly controlled. In particular, for any target divergence $0 \le \delta_\text{alg} < +\infty$ imposed on the overall algorithm, we can always find a proper $\delta \ge 0$ such that $D_f(\vh_\text{alg} \| q) \le  \Gamma(\delta) \le \delta_\text{alg}$. While $\Gamma$ does not admit a closed-form expression, $\delta$ itself is a hyper-parameter. In practice, one can tune $\delta$ to achieve the desired quality-throughput trade-off. This confirms the soundness of our framework.

In fact, our framework also offers a novel theoretical interpretation of typical acceptance sampling.
\begin{restatable}{proposition}{tas}\label{thm:tas}
    Typical acceptance sampling (TAS;~\citeauthor{cai2024medusa},~\citeyear{cai2024medusa}) implicitly solves a variant of the optimization problem in objective~\eqref{def:instance-prob}, where the $f$-divergence is substituted with the cross-entropy $H(\vh, q(\cdot | \vx_{<t}))$.
\end{restatable}
\begin{proof}
    See Appendix~\ref{apx:prf:tas}.
\end{proof}

The suboptimality of TAS arises from the nature of cross-entropy. Specifically, the cross-entropy can be decomposed as
\begin{align}
    H(\vh, q(\cdot| \vx_{<t})) = \underbrace{D_\text{KL}(\vh \| q(\cdot | \vx_{<t}))}_\text{Mode capturing} + \underbrace{H(\vh).}_\text{Certainty}
\end{align}
Here, the KL divergence encourages $h$ to focus on the mode of $q$ (since $h$ is the first argument), while the entropy term encourages $h$ to be deterministic. However, the summation allows $h$ to collapse into a deterministic distribution at the expense of increasing divergence, therefore failing to capture the full shape of $q$. In fact, TAS always yields $h$ with entropy 0 while increasing the divergence by at least $H(q)$. As a result, the produced distribution may diverge significantly from the verifier model, especially when $q$ carries high entropy and thus rich information.

\subsection{\name: constrained acceptance speculative sampling}\label{sec:approach:method}

Based on our analysis above, we propose to use only the KL divergence as the measure of ``distance''. Specifically, this corresponds to the function $f(t) = t \log t$. Combined with our Theorem~\ref{thm:general-solution}, $\gamma^*$ is the root of
\begin{align}
    \Phi(\gamma) :=& q(n | \vx_{<t}) f\left(\frac{\gamma}{q(n | \vx_{<t})}\right) + (1 - q(n | \vx_{<t})) f\left(\frac{1 - \gamma}{1 - q(n | \vx_{<t})}\right) \\
    =& \gamma \log \left(\frac{\gamma}{q(n | \vx_{<t})}\right) + (1 - \gamma) \log\left(\frac{1 - \gamma}{1 - q(n | \vx_{<t})}\right) \\
    =& \delta.
\end{align}

However, since $\Phi$ is a transcendental function involving terms like $x \log x$, it cannot be solved in closed form. We therefore approximate $\Phi$ by its second-order Taylor series expanded at $\gamma_0 = q(n| \vx_{<t})$:
\begin{align}
    \Phi(\gamma) \approx \Phi(\gamma_0) + \Phi'(\gamma_0)(\gamma - \gamma_0) + \frac{\Phi''(\gamma_0)}{2} (\gamma - \gamma_0)^2. \label{eq:approx}
\end{align}
This approximation is justified by noting that $\delta$ is typically small and $\gamma^*$ remains close to $q(n| \vx_{<t})$.

\begin{restatable}[\name's solution]{corollary}{solution}\label{thm:solution}
Let the $f$-divergence in objective~\eqref{def:instance-prob} be the KL divergence. The solution to Equation~\eqref{eq:approx} is given by
\begin{align}\label{eq:sol}
  h(i | \vx_{<t}) = \begin{cases}
      \gamma^*, & \text{if } i=n, \\
      \frac{1 - \gamma^*}{1 - q(n | \vx_{<t})} q(i|\vx_{<t}), & \text{otherwise},
  \end{cases}
\end{align}
where $\gamma^* = \min\Big\{q(n|\vx_{<t}) + \sqrt{2\delta q(n|\vx_{<t})(1 - q(n|\vx_{<t}))}, 1\Big\}$.
\end{restatable}

\begin{proof}
See Appendix~\ref{apx:prf:solution}
\end{proof}

In other words, \name modifies the distribution of the verifier model by increasing the probability of the candidate token $n$ by a small ``bonus'' determined jointly by $q(n|\vx_{<t})$ and $\delta$. We further show that \name's solution is more conservative than the exact solution when the verifier is less confident, ensuring that it strictly satisfies the divergence constraint in such cases.

\begin{restatable}{corollary}{conservative}\label{thm:conservative}
When the exact solution $\gamma^*$ is not greater than $0.5$ (i.e., the token is not likely to be accepted), our approximation always satisfies the divergence constraint:
\begin{align}
    D_\text{KL}(h \| q)  \le \delta,
\end{align}
where $h(n | \vx_{<t})$ is given by the approximated solution in Equation~\eqref{eq:sol}.
\end{restatable}
\begin{proof}
    See Appendix~\ref{apx:prf:conservative}.
\end{proof}

It is easy to see that the bonus probability attains its maximum when $q(n|\vx_{<t}) = 0.5$. In practice, LLMs generally have more than 100K tokens~\citep{dubey2024llama,qwen2024qwen205}, so a probability around $0.5$ indicates strong model confidence in the token. However, SpS could still reject the token $n$ solely because the draft model is overconfident (i.e., $p(n|\vx_{<t})$ is large). Cactus increases the acceptance likelihood in such scenarios by modifying the verifier distribution accordingly.

Compared with TAS's criterion function, \name only requires reading the probability at token $n$ rather than accessing the full vocabulary. This allows \name to further reduce memory access overhead, especially in large-vocabulary settings. More importantly, \name's divergence is tightly controlled with minimal entropy change, whereas TAS yields only low-entropy solutions.

\section{Experiments}\label{sec:exp}
\subsection{Settings}
\paragraph{Datasets.} We evaluated \name on three popular benchmark datasets for large language models: (1) The \textbf{GSM8K}~\citep{cobbe2021training} test set contains 1.3K high-quality grade school math word problems, designed to assess a model's ability to apply mathematics to real-world scenarios. Following common practice in LM-Eval~\citep{eval-harness}, we used 5-shot examples for each test instance. The final accuracy score is averaged over all samples. (2) The \textbf{IFEval}~\citep{zhou2023instruction0following} benchmark measures instruction-following ability. It consists of 500 ``verifiable instructions'' whose outputs can be heuristically evaluated. For example, a prompt might be: ``Write a blog post with 400 or more words about the benefits of sleeping in a hammock,'' which can be automatically checked by counting the number of words. (3) The \textbf{GPQA}~\citep{rein2023gpqa0} diamond benchmark includes approximately 200 challenging science questions authored by domain experts, designed to test models' scientific knowledge. For instance, a sample question is: ``The angular size of the event horizon of a supermassive black hole in the centre of a galaxy at a distance of $d=10^{10}$ parsecs is measured to be $\theta=10^{-17}$ degrees. Find the order of magnitude of the entropy of the black hole.'' Following common practice~\citep{eval-harness}, we format the prompts to include four multiple-choice options. Models are then evaluated by generating a chain-of-thought~\citep{wei2022chain} followed by the final answer.

\paragraph{Evaluation metrics.} For all three tasks, the results are extracted from the generated text by regex matching with the corresponding format. These results are then compared with the gold labels using strict-match accuracy (i.e., $1$ if the strings are identical and $0$ otherwise). Final scores are obtained by averaging the accuracies over all samples. Following previous work~\citep{dubey2024llama}, the regex for GSM8K and GPQA is the ``flexible-extract'' pattern, which selects the first number in the generated sentence regardless of whether the model adheres to the few-shot examples. For IFEval, we use the ``prompt-level-strict-acc'' regex as defined in \cite{qwen2024qwen205}, which requires the model to strictly follow all the instructions.

In addition to task scores, we report the average acceptance length (AL) for all runs. Specifically, AL$_m$ refers to the expected number of accepted tokens among $m$ drafted tokens. A  higher AL$_m$ generally indicates faster generation. However, a method may artificially inflate AL by accepting low-quality draft tokens that lead to longer, less efficient chains of thought. Although AL remains high in this case, the behavior can lead to lower overall throughput due to unnecessarily lengthy outputs. To present a more complete picture of generation efficiency, we also measure the number of rejected tokens during generation, which reflects both the acceptance rate and the total length of generation.

\paragraph{Implementation details.} We used the Qwen~3 series as our main testbed for two reasons: (1) the models come in a variety of sizes, ranging from 0.6B to 14B parameters, enabling a wide range of choices of model pairs; (2) the models are trained to generate with internalized chain-of-thought reasoning~\citep{wei2022chain}, which makes them a natural use case for speculative sampling given the longer generation lengths~\citep{yang2025speculative}. For all experiments, we used the recommended generation parameters~\citep{yang2025qwen3}, where top-$p$ is set to 0.95, top-$k$ is set to $20$, and temperature is $0.6$.

\subsection{Main results}\label{sec:exp:main}
\begin{table}[t]
\caption{The results on three benchmarks: GSM8K, IFEval, and GPQA. We report the ``strict-match'' accuracy as the score with the standard regex pattern for each task. AL$_m$ indicates the number of accepted tokens when the draft length is $m$. Rej denotes the  total number of rejected tokens throughout generation in relative scale, where we use the SpS runs as the reference (labeled as ``Ref'').}\label{tab:main}
\begin{subtable}{\textwidth}
  \subcaption{The results of Qwen~3~8B as verifier and Qwen~3~0.6B as drafter.}
  \label{tab:main:a}
  \centering
  \resizebox{0.9\textwidth}{!}{%
  \begin{tabular}{llccccccccc}
    \toprule
    & & \multicolumn{3}{c}{GSM8K} & \multicolumn{3}{c}{IFEval} & \multicolumn{3}{c}{GPQA} \\
    \cmidrule(r){3-5}  \cmidrule(r){6-8} \cmidrule(r){9-11}
    $m$ & Name      &  Score$^\uparrow$ & AL$_{m}^\uparrow$  & Rej$^\downarrow$  &  Score$^\uparrow$ & AL$_{m}^\uparrow$  & Rej$^\downarrow$  &  Score$^\uparrow$ & AL$_{m}^\uparrow$  & Rej$^\downarrow$  \\
    \midrule 
    & Verifier & \multicolumn{2}{l}{84.31$_{\pm 0.47}$ -} & - &\multicolumn{2}{l}{84.66$_{\pm 0.56}$ -} & - & \multicolumn{2}{l}{41.07$_{\pm 1.77}$ -} & -  \\
    \midrule
    \multirow{4}{*}{10}& SpS  & 83.78 & 4.49 & Ref & 84.66 & 2.59 & Ref & 40.91 & 3.70 & Ref \\
    & TAS  & 86.58 & 5.49 & -29\% & 85.40 & 3.28 & -27\% & 41.41 & 5.17 & -42\% \\
    & \name 0.75 & 85.97 & 5.65 & -34\% & 85.03 & 3.40 & -31\% & 41.42 & 5.33 & -47\% \\
    & \name 1.0 & 86.35  & 5.72 & -37\%  & 84.10  & 3.44 & -32\% & 39.39 & 5.44 & -48\%  \\
    \midrule
    \multirow{4}{*}{20} & SpS  & 84.46 & 5.44 & Ref & 84.10 & 2.74 & Ref & 42.93 & 4.23 & Ref \\
    & TAS  & 85.51 & 7.23 & -35\%  &84.10  & 3.77 &-29\%  & 38.89 & 6.68& -46\% \\
    & \name 0.75 & 86.66 & 7.50 & -37\% & 85.95 & 3.76 & -30\% & 40.01 & 6.89 & -47\% \\
    & \name 1.0 & 86.43 & 7.61 & -39\% & 84.84 & 4.05 & -33\% & 39.90 & 7.05 & -49\% \\
    \bottomrule
  \end{tabular}}
\end{subtable}
\begin{subtable}{\textwidth}
  \vspace{0.75em}
  \subcaption{The results of Qwen~3~14B as verifier and Qwen~3~0.6B as drafter.}
  \label{tab:main:b}
  \centering
    \resizebox{0.9\textwidth}{!}{%
  \begin{tabular}{llccccccccc}
    \toprule
    & & \multicolumn{3}{c}{GSM8K} & \multicolumn{3}{c}{IFEval} & \multicolumn{3}{c}{GPQA} \\
    \cmidrule(r){3-5}  \cmidrule(r){6-8} \cmidrule(r){9-11}
    $m$ & Name      &  Score$^\uparrow$ & AL$_{m}^\uparrow$  & Rej$^\downarrow$  &  Score$^\uparrow$ & AL$_{m}^\uparrow$  & Rej$^\downarrow$  &  Score$^\uparrow$ & AL$_{m}^\uparrow$  & Rej$^\downarrow$  \\
    \midrule 
    & Verifier & \multicolumn{2}{l}{91.71$_{\pm 0.52}$ -} & - & \multicolumn{2}{l}{85.09$_{\pm 0.66}$ -} & - & \multicolumn{2}{l}{40.07$_{\pm 0.77}$ -} & -\\
    \midrule
    \multirow{4}{*}{10}& SpS & 91.12 & 4.27 & Ref & 85.03 & 2.19 & Ref & 39.39& 3.37 & Ref \\
    & TAS  & 92.65 & 5.24 & -30\% &  86.14 & 3.00 & -25\% & 38.89 & 4.99 & -46\% \\
    & \name 0.75 & 92.12 & 5.35 & -31\% & 86.87 & 3.04 & -29\% & 44.95 & 5.14 & -50\% \\
    & \name 1.0 & 93.10 & 5.44 & -32\% & 85.96 & 3.03 & -30\% & 43.43 & 5.16 & -51\% \\
    \midrule
    \multirow{4}{*}{20} & SpS & 91.89 & 5.11 & Ref & 84.47 & 2.27 & Ref & 40.91 & 3.84 & Ref \\
    & TAS  & 92.87 & 6.78 & -32\% & 85.03 & 3.49 & -27\% & 40.40 & 6.41 & -46\% \\
    & \name 0.75 & 92.15 & 7.15 &-36\%  & 86.69 & 3.45 & -30\% & 45.46 & 6.46 & -46\% \\
    & \name 1.0 & 92.87 & 7.00  & -34\% & 86.32 & 3.60 & -30\% & 45.46 & 6.74 & -50\% \\
    \bottomrule
  \end{tabular}}
\end{subtable}
\end{table}

As shown in Table~\ref{tab:main}, speculative sampling (SpS) serves as a strong baseline that closely preserves the output distribution of the verifier model. Across all three benchmarks (GSM8K, IFEval, and GPQA), SpS maintains similar accuracies to the verifier (e.g., 84.46 vs.~84.31 on GSM8K with $m=20$ in Table~\ref{tab:main:a}, and 91.89 vs.~91.71 in Table~\ref{tab:main:b}). This aligns with the theoretical claim that SpS is nearly lossless in generation quality. Additionally, the number of accepted tokens (AL$_m$) for SpS reaches 5.44 on GSM8K and 4.23 on GPQA with $m=20$, indicating that the verifier model is invoked less frequently.

Typical acceptance sampling (TAS) outperforms SpS in terms of acceptance rate, achieving more accepted tokens and lower rejection rates. For example, on GSM8K with $m=20$, TAS improves AL$_m$ from 5.44 to 7.23 (Table~\ref{tab:main:a}) and reduces the rejection rate by 35\%, which is consistent with our approximation analysis in Section~\ref{sec:approach:approx}. However, TAS often introduces distributional shifts that degrade performance. For instance, on GPQA in Table~\ref{tab:main:a}, TAS yields lower accuracy than SpS (38.89 vs. 42.93), likely due to accepting plausible yet suboptimal tokens, especially when the verifier distribution contains fine-grained decision signals.

By contrast, our proposed method, Cactus, achieves the highest acceptance rates across all benchmarks while maintaining or improving accuracy. When $\delta = 0.75$, Cactus consistently surpasses both SpS and TAS in AL$_m$ while achieving strong accuracy, e.g., a score of 86.66 on GSM8K with $m=20$ (Table~\ref{tab:main:a}) and 45.46 on GPQA with $m=20$ (Table~\ref{tab:main:b}), notably outperforming all baselines. When $\delta = 1.0$, Cactus further increases AL$_m$ to 7.61 on GSM8K with a score of 86.43 (Table~\ref{tab:main:a}), or to 7.00 with a score of 92.87 using a larger verifier (Table~\ref{tab:main:b}). Notably, unlike TAS, Cactus does not degrade performance on challenging benchmarks such as GPQA. Instead, it achieves both high acceptance rates and stable accuracy, validating its theoretical foundation in constrained optimization and demonstrating practical robustness across diverse tasks.

\subsection{In-depth analyses}\label{sec:exp:analyses}

\paragraph{Accuracy against acceptance rates.}
\begin{figure}[t]
    \centering
    \includegraphics[width=\linewidth]{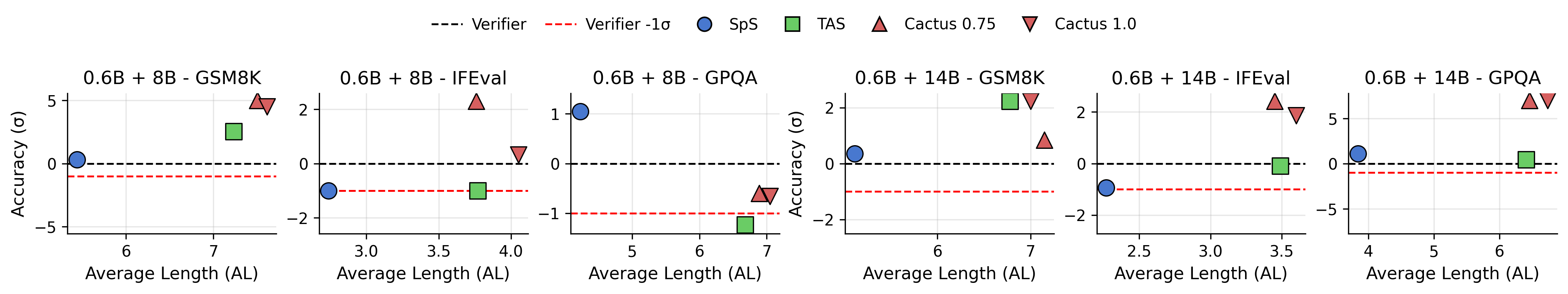}
    \caption{Accuracy-acceptance across benchmarks and model settings. The $x$-axis shows the average accepted length (AL), and the $y$-axis shows accuracy normalized by the standard deviation from the verifier.}
    \label{fig:tradeoff}
\end{figure}

We visualize the accuracy-acceptance trade-off in Figure~\ref{fig:tradeoff}, where accuracy is measured in standard deviations ($\sigma$) from the verifier mean, and throughput is quantified by the average accepted length (AL). Each subplot corresponds to a specific benchmark and verifier-drafter pair. The dashed black line indicates the verifier performance, and the red dashed line marks the $-1\sigma$ threshold, a commonly used indicator of statistically significant degradation.

As shown, TAS improves throughput over SpS but often suffers from accuracy drops, notably falling below the $-1\sigma$ threshold on GPQA with the 8B verifier ($m=20$). In contrast, Cactus consistently preserves accuracy (remaining within or above the verifier confidence range) and frequently exceeds it, such as on GSM8K and IFEval with both 8B and 14B verifiers. This demonstrates that Cactus effectively improves decoding efficiency without compromising (and sometimes even enhancing) generation quality.

It is also worth noting that the improvements from Cactus are stable across tasks with different characteristics. For instance, on challenging benchmarks like GPQA, where other methods either exhibit significant degradation (e.g., TAS) or achieve limited throughput gains (e.g., SpS), Cactus substantially increases AL while maintaining accuracy above baseline. This highlights the strength of our constrained acceptance framework in balancing aggressive token acceptance with distributional fidelity.

\paragraph{The importance of divergence control.}

Our \name dynamically manipulates the target distribution to increase the chance of accepting the sampled tokens. Since this inevitably pushes the target distribution $h$ from the verifier $q$ to be more similar to the draft model $p$, it resembles the notion of mixing the distributions $q$ and $p$ by interpolation. However, we argue that \name is superior to simple interpolation, given that it uses a principled approach which comes with a divergence guarantee. We empirically justify this argument by the following experiment.

\begin{figure}[t]
    \centering
    \includegraphics[width=1.0\linewidth]{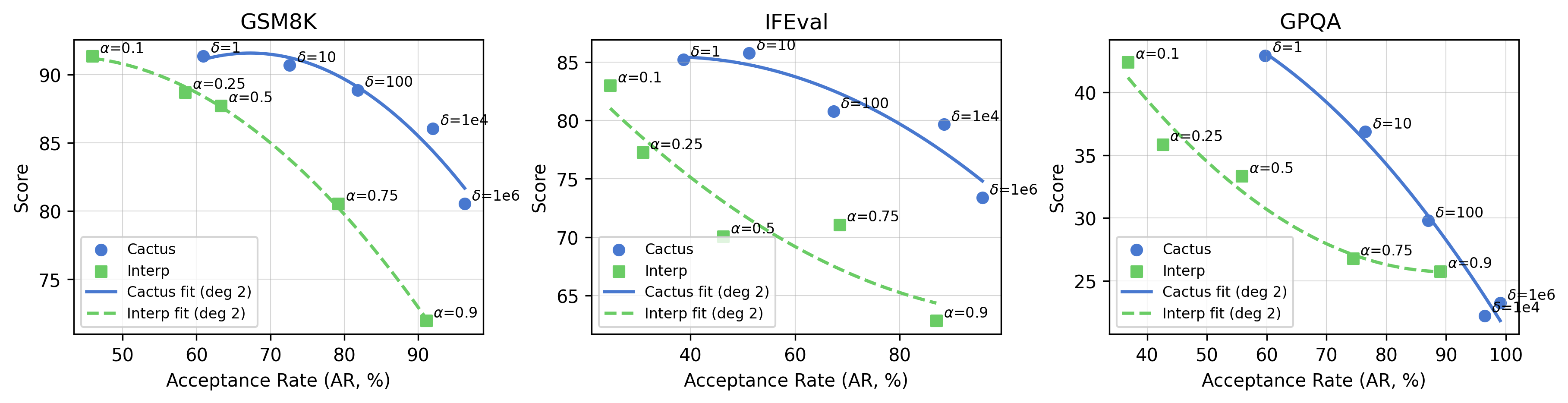}
    \caption{Task score vs.\ acceptance rate for the 0.6B+14B Qwen~3 combination without top-$p$/top-$k$ sampling arguments. Solid and dashed lines are degree-2 polynomial fits.}
    \label{fig:interp}
\end{figure}

Here, we produce data points by running grid search on $\delta$ for \name and interpolation rate $\alpha$ for interpolation, respectively. As shown, \name consistently outperforms interpolation at the similar acceptance rate. For example, at a similar acceptance rate of approximately 90\%, \name achieves a score above 86 ($\delta = 1e4$) on GSM8K, whereas interpolation only reaches a score below 72 ($\alpha=0.9$). Notably, even at a 96.3\% acceptance rate, \name maintains a higher score (above 80), further confirming the necessity of divergence control in our method.

\paragraph{Throughput comparison.}

In Section~\ref{sec:exp:main}, we used AL$_m$ and Rej as proxies for throughput, as they are invariant to hardware and system conditions. Here, we additionally report wall-time results, all measured on A100 40GB GPUs with identical CPU and memory configurations. We used VLLM~\citep{kwon2023efficient} with its default compilation settings to ensure realistic inference conditions. The results on GPQA are shown in Figure~\ref{fig:throughput}.

\begin{figure}[t]
    \centering
    \includegraphics[width=1.0\linewidth]{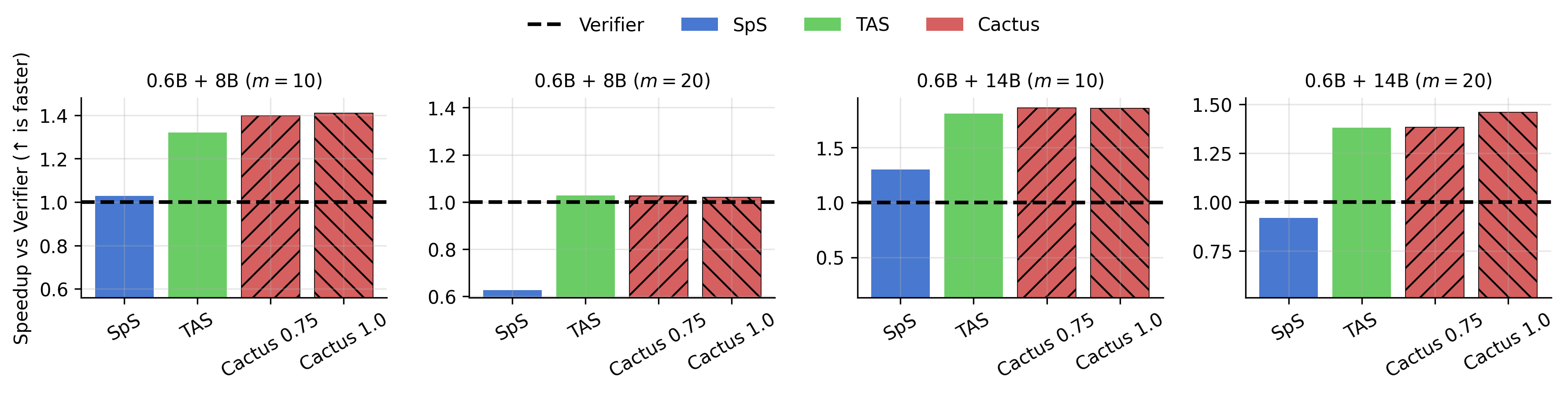}
    \caption{Wall-time normalized throughput ($y$-axis) across different model sizes and draft lengths. The wall time of a single verifier model is always normalized to 1.}
    \label{fig:throughput}
\end{figure}

Across all settings, Cactus remains competitive or superior to all baselines. In particular, Cactus 0.75 and 1.0 yield significant improvements in the 0.6B+14B setting, where Cactus 1.0 achieves nearly 1.9× speedup over the verifier alone with $m=10$, while also maintaining the highest score on GPQA (see Table~\ref{tab:main:b}). In contrast, TAS slightly underperforms Cactus in nearly all settings. Notably, as discussed in Section~\ref{sec:approach:approx} and verified in Table~\ref{tab:main:b}, TAS lacks explicit divergence control. These results highlight the benefit of Cactus's constrained acceptance strategy, which more effectively balances fidelity and efficiency than existing baselines.

\paragraph{Evaluating on more model series.}
To assess the generality of our method, we go beyond Qwen~3 and evaluate three additional model series: Gemma (2B + 9B, \cite{team2024gemma}), DeepSeek~R1 (1.5B + 7B, \cite{deepseek-ai2025deepseekr1}), and LLaMA (1B + 8B, \cite{dubey2024llama}). Each model pair represents a distinct series developed by different teams with varying training methodologies. Following~\cite{bachmann2025judge}, we additionally evaluate top-$k$ decoding as a naive lossy baseline, where draft tokens are accepted if they fall within the top-5 candidates according to the verifier. All drafter-verifier pairs follow the same speculative decoding setup, and accuracy is measured with standard task-specific metrics. We also include SpS and TAS baselines under equivalent configurations to ensure a fair comparison. The results are presented in Figure~\ref{fig:more-models}.

\begin{figure}[t]
    \centering
    \includegraphics[width=0.7\linewidth]{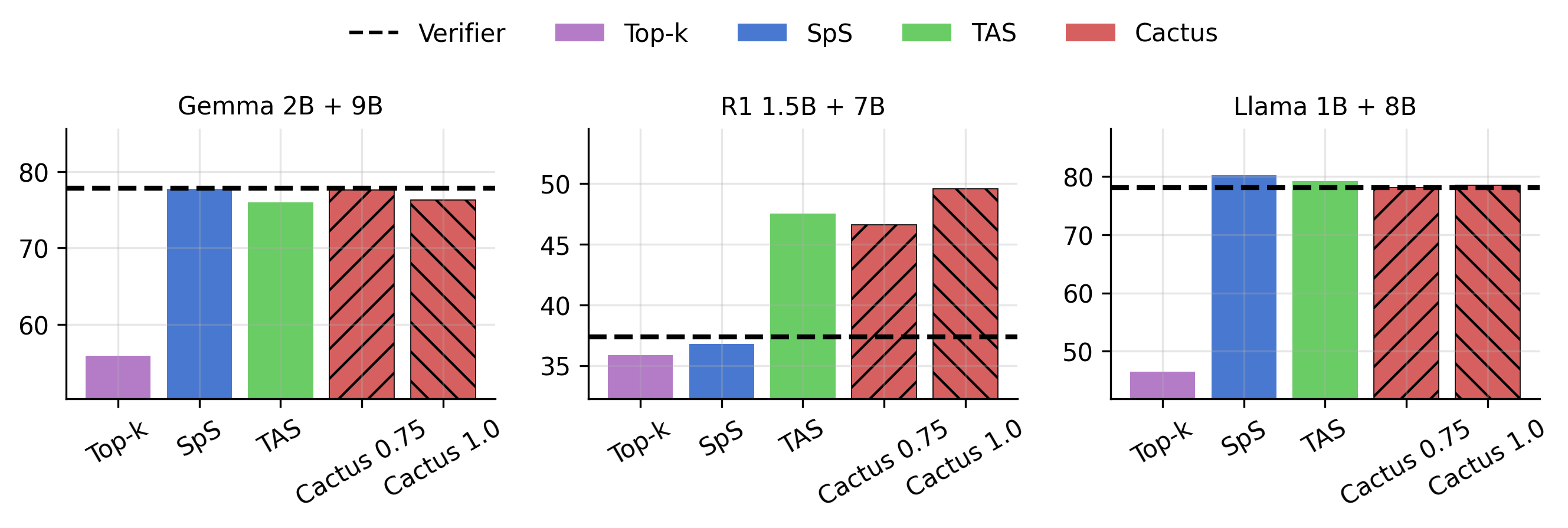}
    \caption{Evaluating on GSM8K with three model pairs.}
    \label{fig:more-models}
\end{figure}

Top-$k$ decoding consistently underperforms the verifier model, reaffirming the importance of using principled verifier-guided sampling like \name. Across all settings, \name delivers strong and consistent performance. For DeepSeek~R1 and Gemma, \name notably outperforms TAS. While SpS and TAS perform well on LLaMA, \name matches their accuracy and retains its robustness across models. These results support the conclusion that \name generalizes well across diverse architectures and remains competitive or superior regardless of the underlying model series.

\paragraph{Scaling to larger models.}

To evaluate the scalability of our method under more memory-intensive conditions, we conduct experiments on a larger model pair: Qwen~3~1.7B (drafter) and 32B (verifier). This setting involves significantly higher parameter counts than the reported 14B maximum in the main table, serving to verify performance where memory bottlenecks are typically more prominent. We maintain the standard speculative decoding setup with a draft length of $m=10$ and report both accuracy and acceptance length (AL).

\begin{table}[t]
\caption{The results of Qwen~3~32B as verifier and Qwen~3~1.7B as drafter on three benchmarks: GSM8K, IFEval, and GPQA. We report the ``strict-match'' accuracy and the acceptance length (AL).}\label{tab:large_models}
\centering
\resizebox{0.6\textwidth}{!}{%
\begin{tabular}{llcccccc}
\toprule
& & \multicolumn{2}{c}{GSM8K} & \multicolumn{2}{c}{IFEval} & \multicolumn{2}{c}{GPQA} \\
\cmidrule(r){3-4} \cmidrule(r){5-6} \cmidrule(r){7-8}
$m$ & Name & Score$^\uparrow$ & AL$_{m}^\uparrow$ & Score$^\uparrow$ & AL$_{m}^\uparrow$ & Score$^\uparrow$ & AL$_{m}^\uparrow$ \\
\midrule
\multirow{3}{*}{10} & SpS & \textbf{95.30} & 5.03 & 83.36 & 2.61 & 40.40 & 3.73 \\
& TAS & 94.10 & 7.02 & 83.73 & 4.16 & 40.40 & 6.12 \\
& Ours ($\delta=1$) & 94.40 & \textbf{7.13} & \textbf{85.21} & \textbf{4.47} & \textbf{41.92} & \textbf{6.36} \\
\bottomrule
\end{tabular}}
\end{table}

As shown in Table~\ref{tab:large_models}, \name demonstrates superior efficiency (achieving the longest acceptance lengths) across all three benchmarks. In terms of task performance, it notably surpasses TAS and SpS on IFEval and GPQA, while achieving a comparable result on GSM8K. These findings confirm that the effectiveness of \name naturally extends to larger models, delivering consistent improvements in acceptance rates while maintaining accuracy.

\section{Related work}\label{sec:related}

\paragraph{The draft-and-verify scheme.}
The most related line of work is the draft-and-verify scheme to accelerate auto-regressive decoding. The foundation of this scheme lies in the acceptance algorithms (i.e., designing the acceptance rate and recovery probability functions in Section~\ref{sec:approach}). This includes vanilla speculative sampling~\citep{chen2023accelerating,leviathan2023fast} and typical acceptance sampling~\citep{hewitt2022truncation,meister2023locally,cai2024medusa}. \name belongs to this category, as it shares the same framework but utilizes a different acceptance strategy. For this reason, we extensively compare it against both methods. In addition to acceptance algorithms, building specialized models for this scheme has been shown to be effective~\citep{kim2023speculative,liu2024kangaroo,liao2025rewardguided}. For instance, \citet{cai2024medusa} fine-tune multiple heads for generating subsequent tokens; \citet{li2024eagle0,li2024eagle020,li2025eagle} propose EAGLE, which introduces an additional head for draft token generation; \citet{bachmann2025judge} propose Judge Decoding, training a binary classifier to augment the acceptance rate function. However, these methods require substantial training resources, whereas \name is a training-free acceptance rule. We also expect that \name can be directly applied to any method that utilizes either SpS or TAS as the underlying principle. Another generalization of speculative sampling involves using multiple draft tokens or verifiers~\citep{yang2024multi,chen2024cascade,jeon2024recursive}. For example, \citet{miao2023specinfer} propose SpecInfer with tree-based draft generation; TreeBoN~\citep{qiu2024treebon} integrates speculative sampling into best-of-N tree-search decoding. We leave the exploration of more integrated versions of multi-drafter or multi-verifier \name to future work.

\paragraph{Low-complexity attention for Transformers.}

Transformer models generate sequences in an auto-regressive manner~\citep{vaswani2017attention}. Since each token attends to all previous ones, generation time grows quadratically with sequence length~\citep{wang2020linformer}. To address this, previous work has proposed low-complexity attention variants~\citep{child2019generating,zaheer2020big,tsai2019transformer,kitaev2020reformer,choromanski2021rethinking}. These methods modify the Transformer architecture itself. \name can be combined with these methods since they also follow the auto-regressive paradigm. In addition to architectural changes, decoding complexity can also be reduced by manipulating the KV cache~\citep{zhang2024h2o,li2024snapkv,cai2024pyramidkv}. For instance, SnapKV~\citep{li2024snapkv} evicts less relevant tokens from the prompt before generation; \citet{hao2025radar} and~\citet{wu2026tokmem} compress context with a more efficient representation. These techniques are orthogonal to speculative sampling methods like \name.

\paragraph{Minimizing overheads of Transformers.}

Without approximating the Transformer architecture, overheads can still be reduced to accelerate decoding. Flash Attention~\citep{dao2022flashattention,dao2023flashattention}, for example, uses tiling techniques to avoid memory-bound operations, and has seen widespread adoption~\citep{wolf2019huggingface0s,kwon2023efficient}. Memory-efficient attention~\citep{rabe2021self} reorders computation to maintain constant memory usage regardless of context length. Another line of work applies quantization to model parameters~\citep{lin2023awq,badri2023hqq,liu2024spinquant0}. The benefits are threefold: (1) reduced memory footprint due to lower-precision data types; (2) alleviated memory bottlenecks during decoding; and (3) improved hardware efficiency via optimized kernels. All these methods can be seamlessly integrated into speculative sampling approaches, including \name.

\section{Conclusion}\label{sec:conclusion}

In this paper, we propose a constrained optimization framework for analyzing and improving speculative sampling methods. Building upon this framework, we introduce \name, a novel training-free speculative sampling method that increases acceptance rates while maintaining a provably controlled divergence from the large verifier model. \name uses only basic element-wise operations, making it highly practical and lightweight for real-time inference. We empirically evaluate our method on a variety of benchmarks and confirm its effectiveness. As LLMs continue to grow in size and cost, our method provides a theoretically grounded yet practically efficient solution for scalable deployment.

\section*{Acknowledgments}
We thank the reviewers and chairs for their efforts. The research is supported in part by the Natural Sciences and Engineering Research Council of Canada (NSERC), the Amii Fellow Program, the Canada CIFAR AI Chair Program, a donation from DeepMind, and the Digital Research Alliance of Canada (alliancecan.ca).

\section*{Ethics statement}

We certify that all authors of this project adhere to the ICLR Code of Ethics (\url{https://iclr.cc/public/CodeOfEthics}). Our research does not involve human subjects, practices related to dataset releases, potentially harmful content, potential conflicts of interest and sponsorship, discrimination/bias/fairness concerns, privacy and security issues, legal compliance, or research integrity issues.

\section*{Reproducibility statement}

All of our experiments use publicly accessible datasets and models. Specifically, the datasets we used can be found at the following links via HuggingFace:
\begin{itemize}
    \item GSM8K: \href{https://huggingface.co/datasets/openai/gsm8k}{https://huggingface.co/datasets/openai/gsm8k}
    \item IFEval: \href{https://huggingface.co/datasets/google/IFEval}{https://huggingface.co/datasets/google/IFEval}
    \item GPQA: \href{https://huggingface.co/datasets/Idavidrein/gpqa}{https://huggingface.co/datasets/Idavidrein/gpqa}
\end{itemize} 

The models can be found at the following links:
\begin{itemize}
    \item Qwen~3~0.6B:~\href{https://huggingface.co/Qwen/Qwen3-0.6B}{https://huggingface.co/Qwen/Qwen3-0.6B}
    \item Qwen~3~1.7B:~\href{https://huggingface.co/Qwen/Qwen3-1.7B}{https://huggingface.co/Qwen/Qwen3-1.7B}
    \item Qwen~3~4B:~\href{https://huggingface.co/Qwen/Qwen3-4B}{https://huggingface.co/Qwen/Qwen3-4B}
    \item Qwen~3~8B:~\href{https://huggingface.co/Qwen/Qwen3-8B}{https://huggingface.co/Qwen/Qwen3-8B}
    \item Qwen~3~14B:~\href{https://huggingface.co/Qwen/Qwen3-14B}{https://huggingface.co/Qwen/Qwen3-14B}
    \item Qwen~3~32B:~\href{https://huggingface.co/Qwen/Qwen3-32B}{https://huggingface.co/Qwen/Qwen3-32B}
    \item Gemma~2B:~\href{https://huggingface.co/google/gemma-2-2b}{https://huggingface.co/google/gemma-2-2b}
    \item Gemma~9B:~\href{https://huggingface.co/google/gemma-2-9b}{https://huggingface.co/google/gemma-2-9b}
    \item DeepSeek~R1~1.5B:~\href{https://huggingface.co/deepseek-ai/DeepSeek-R1-Distill-Qwen-1.5B}{https://huggingface.co/deepseek-ai/DeepSeek-R1-Distill-Qwen-1.5B}
    \item DeepSeek~R1~7B:~\href{https://huggingface.co/deepseek-ai/DeepSeek-R1-Distill-Qwen-7B}{https://huggingface.co/deepseek-ai/DeepSeek-R1-Distill-Qwen-7B}
    \item LLaMA~1B:~\href{https://huggingface.co/meta-llama/Llama-3.2-1B}{https://huggingface.co/meta-llama/Llama-3.2-1B}
    \item LLaMA~8B:~\href{https://huggingface.co/meta-llama/Llama-3.1-8B}{https://huggingface.co/meta-llama/Llama-3.1-8B}
\end{itemize}

In addition, our code is publicly available at \href{https://github.com/MANGA-UOFA/Cactus}{https://github.com/MANGA-UOFA/Cactus}.

\bibliographystyle{plainnat}
\bibliography{main}

@inproceedings{xia2024unlocking,
    title = "Unlocking Efficiency in Large Language Model Inference: A Comprehensive Survey of Speculative Decoding",
    author = "Xia, Heming and Yang, Zhe and Dong, Qingxiu and Wang, Peiyi and Li, Yongqi  and Ge, Tao and Liu, Tianyu and Li, Wenjie and Sui, Zhifang",
    booktitle = "Findings of ACL",
    year = "2024",
    url = "https://aclanthology.org/2024.findings-acl.456",
    pages = "7655--7671",
}

@inproceedings{hu2021lora,
  title     = {{LoRA}: Low-Rank Adaptation of Large Language Models},
  author    = {Edward J. Hu and Yelong Shen and Phillip Wallis and Zeyuan Allen-Zhu and Yuanzhi Li and Shean Wang and Weizhu Chen},
  booktitle = {ICLR},
  year      = {2022},
  url       = {https://openreview.net/forum?id=nZeVKeeFYf9}
}

@inproceedings{wen2023f,
  title={f-divergence minimization for sequence-level knowledge distillation},
  author={Wen, Yuqiao and Li, Zichao and Du, Wenyu and Mou, Lili},
  booktitle={ACL},
  pages={10817--10834},
  year={2023},
  url={https://aclanthology.org/2023.acl-long.605}
}

@inproceedings{wu2025ulpt,
  title={Ultra-Low-Dimensional Prompt Tuning via Random Projection},
  author={Wu, Zijun and Hao, Yongchang and Mou, Lili},
  booktitle={EACL},
  year={2026},
  url={https://openreview.net/forum?id=hb8Xm3kRCY}
}

@inproceedings{
wu2026tokmem,
title={{TokMem}: Tokenized Procedural Memory for Large Language Models},
author={Zijun Wu and Yongchang Hao and Lili Mou},
booktitle={ICLR},
year={2026},
url={https://openreview.net/forum?id=RWjEf9PdiJ}
}

@inproceedings{jeon2024recursive,
  title     = {Recursive Speculative Decoding: Accelerating {LLM} Inference via Sampling Without Replacement},
  author    = {Wonseok Jeon and Mukul Gagrani and Raghavv Goel and Junyoung Park and Mingu Lee and Christopher Lott},
  year      = {2024},
  booktitle = {ICLR Workshop on Large Language Model (LLM) Agents},
  url       = {https://arxiv.org/abs/2402.14160},
}

@inproceedings{
li2025eagle,
title={{EAGLE}-3: Scaling up Inference Acceleration of Large Language Models via Training-Time Test},
author={Yuhui Li and Fangyun Wei and Chao Zhang and Hongyang Zhang},
booktitle={NeurIPS},
year={2025},
url={https://openreview.net/forum?id=4exx1hUffq}
}

@inproceedings{stern2018blockwise,
  title={Blockwise parallel decoding for deep autoregressive models},
  author={Stern, Mitchell and Shazeer, Noam and Uszkoreit, Jakob},
  booktitle={NeurIPS},
  year={2018},
  url={https://proceedings.neurips.cc/paper/2018/hash/c4127b9194fe8562c64dc0f5bf2c93bc-Abstract.html}
}

@article{chen2023accelerating,
  title={Accelerating large language model decoding with speculative sampling},
  author={Chen, Charlie and Borgeaud, Sebastian and Irving, Geoffrey and Lespiau, Jean-Baptiste and Sifre, Laurent and Jumper, John},
  journal={arXiv preprint arXiv:2302.01318},
  year={2023},
  url={https://arxiv.org/abs/2302.01318}
}

@article{kaplan2020scaling,
  title={Scaling laws for neural language models},
  author={Kaplan, Jared and McCandlish, Sam and Henighan, Tom and Brown, Tom B and Chess, Benjamin and Child, Rewon and Gray, Scott and Radford, Alec and Wu, Jeffrey and Amodei, Dario},
  journal={arXiv preprint arXiv:2001.08361},
  year={2020},
  url={https://arxiv.org/abs/2001.08361}
}

@inproceedings{li2024eagle0,
  title     = {{EAGLE}: Speculative Sampling Requires Rethinking Feature Uncertainty},
  author    = {Yuhui Li and Fangyun Wei and Chao Zhang and Hongyang Zhang},
  booktitle = {ICML},
  year      = {2024},
  url       = {https://proceedings.mlr.press/v235/li24bt.html}
}

@inproceedings{
bachmann2025judge,
title={Judge Decoding: Faster Speculative Sampling Requires Going Beyond Model Alignment},
author={Gregor Bachmann and Sotiris Anagnostidis and Albert Pumarola and Markos Georgopoulos and Artsiom Sanakoyeu and Yuming Du and Edgar Sch{\"o}nfeld and Ali Thabet and Jonas K Kohler},
booktitle={ICLR},
year={2025},
url={https://openreview.net/forum?id=mtSSFiqW6y}
}

@inproceedings{li2024eagle020,
  title     = {{EAGLE}-2: Faster Inference of Language Models with Dynamic Draft Trees},
  author    = {Yuhui Li and Fangyun Wei and Chao Zhang and Hongyang Zhang},
  year      = {2024},
  booktitle = {EMNLP},
  url       = {https://aclanthology.org/2024.emnlp-main.422}
}

@inproceedings{holtzman2019curious,
  title={The curious case of neural text degeneration},
  author={Holtzman, Ari and Buys, Jan and Du, Li and Forbes, Maxwell and Choi, Yejin},
  booktitle={ICLR},
  year={2020},
  url={https://openreview.net/forum?id=rygGQyrFvH}
}

@article{yuan2024llm,
  title   = {{LLM} Inference Unveiled: Survey and Roofline Model Insights},
  author  = {Zhihang Yuan and Yuzhang Shang and Yang Zhou and Zhen Dong and Zhe Zhou and Chenhao Xue and Bingzhe Wu and Zhikai Li and Qingyi Gu and Yong Jae Lee and Yan Yan and Beidi Chen and Guangyu Sun and Kurt Keutzer},
  year    = {2024},
  journal = {arXiv preprint arXiv:2402.16363},
  url     = {https://arxiv.org/abs/2402.16363}
}

@inproceedings{hewitt2022truncation,
  title     = {Truncation Sampling as Language Model Desmoothing},
  author    = {John Hewitt and Christopher D. Manning and Percy Liang},
  year      = {2022},
  booktitle = {Findings of EMNLP},
  url       = {https://aclanthology.org/2022.findings-emnlp.249}
}

@inproceedings{meister2020if,
  title={If beam search is the answer, what was the question?},
  author={Meister, Clara and Vieira, Tim and Cotterell, Ryan},
  booktitle={EMNLP},
  year={2020},
  url={https://aclanthology.org/2020.emnlp-main.170}
}

@inproceedings{xia2022speculative,
  title={Speculative decoding: Exploiting speculative execution for accelerating seq2seq generation},
  author={Xia, Heming and Ge, Tao and Wang, Peiyi and Chen, Si-Qing and Wei, Furu and Sui, Zhifang},
  booktitle={Findings of EMNLP},
  year={2023},
  url={https://aclanthology.org/2023.findings-emnlp.257}
}

@article{meister2023locally,
  title={Locally typical sampling},
  author={Meister, Clara and Pimentel, Tiago and Wiher, Gian and Cotterell, Ryan},
  journal={TACL},
  volume={11},
  pages={102--121},
  year={2023},
  publisher={MIT Press},
  url={https://doi.org/10.1162/tacl_a_00536}
}

@inproceedings{leviathan2023fast,
  title={Fast inference from {T}ransformers via speculative decoding},
  author={Leviathan, Yaniv and Kalman, Matan and Matias, Yossi},
  booktitle={ICML},
  pages={19274--19286},
  year={2023},
  url={https://proceedings.mlr.press/v202/leviathan23a.html}
}

@misc{eval-harness,
  author       = {Gao, Leo and Tow, Jonathan and Abbasi, Baber and Biderman, Stella and Black, Sid and DiPofi, Anthony and Foster, Charles and Golding, Laurence and Hsu, Jeffrey and Le Noac'h, Alain and Li, Haonan and McDonell, Kyle and Muennighoff, Niklas and Ociepa, Chris and Phang, Jason and Reynolds, Laria and Schoelkopf, Hailey and Skowron, Aviya and Sutawika, Lintang and Tang, Eric and Thite, Anish and Wang, Ben and Wang, Kevin and Zou, Andy},
  title        = {The Language Model Evaluation Harness},
  year         = 2024,
  publisher    = {Zenodo},
  version      = {v0.4.3},
  url          = {https://zenodo.org/records/12608602}
}

@article{qwen2024qwen205,
  title   = {Qwen2.5 Technical Report},
  author  = {Qwen Team, {}},
  year    = {2024},
  journal = {arXiv preprint arXiv:2412.15115},
  url = {https://arxiv.org/abs/2412.15115}
}

@article{dubey2024llama,
  title   = {The {L}lama 3 Herd of Models},
  author  = {Abhimanyu Dubey and Abhinav Jauhri and Abhinav Pandey and Abhishek Kadian and Ahmad Al-Dahle and Aiesha Letman and Akhil Mathur and Alan Schelten and Amy Yang and Angela Fan and Anirudh Goyal and Anthony Hartshorn and Aobo Yang and Archi Mitra and Archie Sravankumar and et al.},
  year    = {2024},
  journal = {arXiv preprint arXiv:2407.21783},
  url = {https://arxiv.org/abs/2407.21783},
}

@inproceedings{dao2023flashattention,
  title={Flash{A}ttention-2: Faster attention with better parallelism and work partitioning},
  author={Dao, Tri},
  booktitle={ICLR},
  year={2024},
  url={https://openreview.net/forum?id=mZn2Xyh9Ec},
}

@inproceedings{dao2022flashattention,
  title={Flash{A}ttention: Fast and memory-efficient exact attention with {IO}-awareness},
  author={Dao, Tri and Fu, Dan and Ermon, Stefano and Rudra, Atri and R{\'e}, Christopher},
  booktitle={NeurIPS},
  pages={16344--16359},
  year={2022},
  url={https://proceedings.neurips.cc/paper_files/paper/2022/hash/67d57c32e20fd0a7a302cb81d36e40d5-Abstract-Conference.html},
}

@article{rabe2021self,
  title   = {Self-attention Does Not Need $O(n^2)$ Memory},
  author  = {Markus N. Rabe and Charles Staats},
  year    = {2021},
  journal = {arXiv preprint arXiv:2112.05682},
  url = {https://arxiv.org/abs/2112.05682},
}

@article{team2024gemma,
  title={Gemma: Open models based on {G}emini research and technology},
  author={Gemma Team, {}},
  journal={arXiv preprint arXiv:2403.08295},
  year={2024},
  url = {https://arxiv.org/abs/2403.08295}
}

@inproceedings{liao2025rewardguided,
  title     = {Reward-Guided Speculative Decoding for Efficient {LLM} Reasoning},
  author    = {Baohao Liao and Yuhui Xu and Hanze Dong and Junnan Li and Christof Monz and Silvio Savarese and Doyen Sahoo and Caiming Xiong},
  year      = {2025},
  booktitle = {ICML},
  url       = {https://openreview.net/forum?id=AVeskAAETB}
}

@article{deepseek-ai2025deepseekr1,
  title   = {{DeepSeek}-{R}1: Incentivizing Reasoning Capability in {LLM}s via Reinforcement Learning},
  author       = {Daya Guo and Dejian Yang and Haowei Zhang and Junxiao Song and Peiyi Wang and Qihao Zhu and Runxin Xu and Ruoyu Zhang and Shirong Ma and Xiao Bi and Xiaokang Zhang and Xingkai Yu and Yu Wu and Z. F. Wu and Zhibin Gou and Zhihong Shao and Zhuoshu Li and Ziyi Gao and Aixin Liu and Bing Xue and Bingxuan Wang and Bochao Wu and Bei Feng and Chengda Lu and Chenggang Zhao and Chengqi Deng and Chong Ruan and Damai Dai and Deli Chen and Dongjie Ji and Erhang Li and Fangyun Lin and Fucong Dai and Fuli Luo and Guangbo Hao and Guanting Chen and Guowei Li and Hao Zhang and Hanwei Xu and Honghui Ding and Huazuo Gao and Hui Qu and Hui Li and Jianzhong Guo and Jiashi Li and Jingchang Chen and Jingyang Yuan and Jinhao Tu and Junjie Qiu and Junlong Li and J. L. Cai and Jiaqi Ni and Jian Liang and Jin Chen and Kai Dong and Kai Hu and Kaichao You and Kaige Gao and Kang Guan and Kexin Huang and Kuai Yu and Lean Wang and Lecong Zhang and Liang Zhao and Litong Wang and Liyue Zhang and Lei Xu and Leyi Xia and Mingchuan Zhang and Minghua Zhang and Minghui Tang and Mingxu Zhou and Meng Li and Miaojun Wang and Mingming Li and Ning Tian and Panpan Huang and Peng Zhang and Qiancheng Wang and Qinyu Chen and Qiushi Du and Ruiqi Ge and Ruisong Zhang and Ruizhe Pan and Runji Wang and R. J. Chen and R. L. Jin and Ruyi Chen and Shanghao Lu and Shangyan Zhou and Shanhuang Chen and Shengfeng Ye and Shiyu Wang and Shuiping Yu and Shunfeng Zhou and Shuting Pan and S. S. Li and Shuang Zhou and Shaoqing Wu and Tao Yun and Tian Pei and Tianyu Sun and Tao Wang and Wangding Zeng and Wen Liu and Wenfeng Liang and Wenjun Gao and Wenqin Yu and Wentao Zhang and W. L. Xiao and Wei An and Xiaodong Liu and Xiaohan Wang and Xiaokang Chen and Xiaotao Nie and Xin Cheng and Xin Liu and Xin Xie and Xingchao Liu and Xinyu Yang and Xinyuan Li and Xuecheng Su and Xuheng Lin and X. Q. Li and Xiangyue Jin and Xiaojin Shen and Xiaosha Chen and Xiaowen Sun and Xiaoxiang Wang and Xinnan Song and Xinyi Zhou and Xianzu Wang and Xinxia Shan and Y. K. Li and Y. Q. Wang and Y. X. Wei and Yang Zhang and Yanhong Xu and Yao Li and Yao Zhao and Yaofeng Sun and Yaohui Wang and Yi Yu and Yichao Zhang and Yifan Shi and Yiliang Xiong and Ying He and Yishi Piao and Yisong Wang and Yixuan Tan and Yiyang Ma and Yiyuan Liu and Yongqiang Guo and Yuan Ou and Yuduan Wang and Yue Gong and Yuheng Zou and Yujia He and Yunfan Xiong and Yuxiang Luo and Yuxiang You and Yuxuan Liu and Yuyang Zhou and Y. X. Zhu and Yanping Huang and Yaohui Li and Yi Zheng and Yuchen Zhu and Yunxian Ma and Ying Tang and Yukun Zha and Yuting Yan and Z. Z. Ren and Zehui Ren and Zhangli Sha and Zhe Fu and Zhean Xu and Zhenda Xie and Zhengyan Zhang and Zhewen Hao and Zhicheng Ma and Zhigang Yan and Zhiyu Wu and Zihui Gu and Zijia Zhu and Zijun Liu and Zilin Li and Ziwei Xie and Ziyang Song and Zizheng Pan and Zhen Huang and Zhipeng Xu and Zhongyu Zhang and Zhen Zhang},
  year    = {2025},
  journal = {Nature},
  volume  = {645},
  pages   = {633--638},
  url     = {https://www.nature.com/articles/s41586-025-09422-z}
}

@inproceedings{hu2025speculative,
  title={Speculative Decoding and Beyond: An In-Depth Survey of Techniques},
  author={Hu, Yunhai and Liu, Zining and Dong, Zhenyuan and Peng, Tianfan and McDanel, Bradley and Zhang, Sai Qian},
  booktitle={Findings of EMNLP},
  year={2025},
  url={https://aclanthology.org/2025.findings-emnlp.716}
}

@article{zhou2024survey,
  title={A survey on efficient inference for large language models},
  author={Zhou, Zixuan and Ning, Xuefei and Hong, Ke and Fu, Tianyu and Xu, Jiaming and Li, Shiyao and Lou, Yuming and Wang, Luning and Yuan, Zhihang and Li, Xiuhong and Yan, Shengen and Dai, Guohao and Zhang, Xiao-Ping and Dong, Yuhan and Wang, Yu},
  journal={arXiv preprint arXiv:2404.14294},
  year={2024},
  url = {https://arxiv.org/abs/2404.14294}
}

@article{child2019generating,
  title={Generating long sequences with sparse {T}ransformers},
  author={Child, Rewon and Gray, Scott and Radford, Alec and Sutskever, Ilya},
  journal={arXiv preprint arXiv:1904.10509},
  year={2019},
  url={https://arxiv.org/abs/1904.10509}
}

@inproceedings{zaheer2020big,
  title={Big {B}ird: {T}ransformers for longer sequences},
  author={Zaheer, Manzil and Guruganesh, Guru and Dubey, Kumar Avinava and Ainslie, Joshua and Alberti, Chris and Ontanon, Santiago and Pham, Philip and Ravula, Anirudh and Wang, Qifan and Yang, Li and Ahmed, Amr},
  booktitle={NeurIPS},
  pages={17283--17297},
  year={2020},
  url={https://proceedings.neurips.cc/paper/2020/hash/c8512d142a2d849725f31a9a7a361ab9-Abstract.html}
}

@inproceedings{Vaswani2017attention,
 author = {Vaswani, Ashish and Shazeer, Noam and Parmar, Niki and Uszkoreit, Jakob and Jones, Llion and Gomez, Aidan N and Kaiser, {\L}ukasz and Polosukhin, Illia},
 booktitle = {NIPS},
 title = {Attention is All you Need},
 url = {https://proceedings.neurips.cc/paper_files/paper/2017/file/3f5ee243547dee91fbd053c1c4a845aa-Paper.pdf},
 year = {2017}
}

@inproceedings{zhang2024h2o,
  title={H$_2$O: Heavy-hitter oracle for efficient generative inference of large language models},
  author={Zhang, Zhenyu and Sheng, Ying and Zhou, Tianyi and Chen, Tianlong and Zheng, Lianmin and Cai, Ruisi and Song, Zhao and Tian, Yuandong and R{\'e}, Christopher and Barrett, Clark and Wang, Zhangyang and Chen, Beidi},
  booktitle={NeurIPS},
  year={2023},
  url={https://proceedings.neurips.cc/paper_files/paper/2023/hash/6ceefa7b15572587b78ecfcebb2827f8-Abstract-Conference.html}
}

@inproceedings{li2024snapkv,
  title={Snap{KV}: {LLM} knows what you are looking for before generation},
  author={Li, Yuhong and Huang, Yingbing and Yang, Bowen and Venkitesh, Bharat and Locatelli, Acyr and Ye, Hanchen and Cai, Tianle and Lewis, Patrick and Chen, Deming},
  booktitle={NeurIPS},
  year={2024},
  url={https://proceedings.neurips.cc/paper_files/paper/2024/hash/28ab418242603e0f7323e54185d19bde-Abstract-Conference.html}
}

@inproceedings{cai2024pyramidkv,
  title     = {Pyramid{KV}: Dynamic {KV} Cache Compression based on Pyramidal Information Funneling},
  author    = {Zefan Cai and Yichi Zhang and Bofei Gao and Yuliang Liu and Tianyu Liu and Keming Lu and Wayne Xiong and Yue Dong and Baobao Chang and Junjie Hu and Wen Xiao},
  year      = {2025},
  booktitle = {COLM},
  url       = {https://openreview.net/forum?id=ayi7qezU87}
}

@inproceedings{
hao2025radar,
title={Radar: Fast Long-Context Decoding for Any {T}ransformers},
author={Yongchang Hao and Mengyao Zhai and Hossein Hajimirsadeghi and Sepidehsadat Hosseini and Frederick Tung},
booktitle={ICLR},
year={2025},
url={https://openreview.net/forum?id=ZTpWOwMrzQ}
}

@inproceedings{wolf2019huggingface0s,
  title     = {{T}ransformers: State-of-the-art Natural Language Processing},
  author    = {Thomas Wolf and Lysandre Debut and Victor Sanh and Julien Chaumond and Clement Delangue and Anthony Moi and Pierric Cistac and Tim Rault and R{\'e}mi Louf and Morgan Funtowicz and Joe Davison and Sam Shleifer and Patrick von Platen and Clara Ma and Yacine Jernite and Julien Plu and Canwen Xu and Teven Le Scao and Sylvain Gugger and Mariama Drame and Quentin Lhoest and Alexander M. Rush},
  year      = {2020},
  booktitle = {EMNLP},
  url       = {https://aclanthology.org/2020.emnlp-demos.6},
}

@article{cobbe2021training,
  title   = {Training Verifiers to Solve Math Word Problems},
  author  = {Karl Cobbe and Vineet Kosaraju and Mohammad Bavarian and Mark Chen and Heewoo Jun and Lukasz Kaiser and Matthias Plappert and Jerry Tworek and Jacob Hilton and Reiichiro Nakano and Christopher Hesse and John Schulman},
  year    = {2021},
  journal = {arXiv preprint arXiv:2110.14168},
  url = {https://arxiv.org/abs/2110.14168}
}

@article{zhou2023instruction0following,
  title   = {Instruction-Following Evaluation for Large Language Models},
  author  = {Jeffrey Zhou and Tianjian Lu and Swaroop Mishra and Siddhartha Brahma and Sujoy Basu and Yi Luan and Denny Zhou and Le Hou},
  year    = {2023},
  journal = {arXiv preprint arXiv:2311.07911},
  url = {https://arxiv.org/abs/2311.07911}
}

@inproceedings{rein2023gpqa0,
  title     = {{GPQA}: A Graduate-Level {G}oogle-Proof {Q}\&{A} Benchmark},
  author    = {David Rein and Betty Li Hou and Asa Cooper Stickland and Jackson Petty and Richard Yuanzhe Pang and Julien Dirani and Julian Michael and Samuel R. Bowman},
  year      = {2024},
  booktitle = {COLM},
  url       = {https://openreview.net/forum?id=Ti67584b98}
}

@article{yang2025qwen3,
  title   = {Qwen3 Technical Report},
  author  = {An Yang and Anfeng Li and Baosong Yang and Beichen Zhang and Binyuan Hui and Bo Zheng and Bowen Yu and Chang Gao and Chengen Huang and Chenxu Lv and Chujie Zheng and Dayiheng Liu and Fan Zhou and Fei Huang and Feng Hu and Hao Ge and Haoran Wei and Huan Lin and Jialong Tang and Jian Yang and Jianhong Tu and Jianwei Zhang and Jianxin Yang and Jiaxi Yang and Jing Zhou and Jingren Zhou and Junyang Lin and Kai Dang and Keqin Bao and Kexin Yang and Le Yu and Lianghao Deng and Mei Li and Mingfeng Xue and Mingze Li and Pei Zhang and Peng Wang and Qin Zhu and Rui Men and Ruize Gao and Shixuan Liu and Shuang Luo and Tianhao Li and Tianyi Tang and Wenbiao Yin and Xingzhang Ren and Xinyu Wang and Xinyu Zhang and Xuancheng Ren and Yang Fan and Yang Su and Yichang Zhang and Yinger Zhang and Yu Wan and Yuqiong Liu and Zekun Wang and Zeyu Cui and Zhenru Zhang and Zhipeng Zhou and Zihan Qiu},
  year    = {2025},
  journal = {arXiv preprint arXiv:2505.09388},
  url     = {https://arxiv.org/abs/2505.09388}
}

@inproceedings{yang2025speculative,
  title     = {Speculative Thinking: Enhancing Small-Model Reasoning with Large Model Guidance at Inference Time},
  author    = {Wang Yang and Xiang Yue and Vipin Chaudhary and Xiaotian Han},
  year      = {2025},
  booktitle = {COLM},
  url       = {https://openreview.net/forum?id=4Ns18bSoHo}
}

@inproceedings{
wei2022chain,
title={Chain of Thought Prompting Elicits Reasoning in Large Language Models},
author={Jason Wei and Xuezhi Wang and Dale Schuurmans and Maarten Bosma and brian ichter and Fei Xia and Ed H. Chi and Quoc V Le and Denny Zhou},
booktitle={NeurIPS},
year={2022},
url={https://openreview.net/forum?id=_VjQlMeSB_J}
}

@inproceedings{liu2024spinquant0,
  title     = {Spin{Q}uant: {LLM} quantization with learned rotations},
  author    = {Zechun Liu and Changsheng Zhao and Igor Fedorov and Bilge Soran and Dhruv Choudhary and Raghuraman Krishnamoorthi and Vikas Chandra and Yuandong Tian and Tijmen Blankevoort},
  year      = {2025},
  booktitle = {ICLR},
  url       = {https://openreview.net/forum?id=ogO6DGE6FZ}
}

@misc{badri2023hqq,
	title = {Half-Quadratic Quantization of Large Machine Learning Models},
	url = {https://dropbox.github.io/hqq_blog/},
	howpublished = {Online manuscript},
	author = {Hicham Badri and Appu Shaji},
	year = {2023}
}

@inproceedings{lin2023awq,
  title     = {{AWQ}: Activation-aware Weight Quantization for On-Device {LLM} Compression and Acceleration},
  author    = {Ji Lin and Jiaming Tang and Haotian Tang and Shang Yang and Wei-Ming Chen and Wei-Chen Wang and Guangxuan Xiao and Xingyu Dang and Chuang Gan and Song Han},
  year      = {2024},
  booktitle = {MLSys},
  url       = {https://proceedings.mlsys.org/paper_files/paper/2024/hash/42a452cbafa9dd64e9ba4aa95cc1ef21-Abstract-Conference.html}
}

@inproceedings{kwon2023efficient,
  title={Efficient Memory Management for Large Language Model Serving with {P}aged{A}ttention},
  author={Woosuk Kwon and Zhuohan Li and Siyuan Zhuang and Ying Sheng and Lianmin Zheng and Cody Hao Yu and Joseph E. Gonzalez and Hao Zhang and Ion Stoica},
  booktitle={SOSP},
  year={2023},
  url={https://dl.acm.org/doi/10.1145/3600006.3613165}
}

@inproceedings{tsai2019transformer,
    title = "Transformer Dissection: An Unified Understanding for {T}ransformers{'}s Attention via the Lens of Kernel",
    author = "Tsai, Yao-Hung Hubert  and Bai, Shaojie  and Yamada, Makoto  and Morency, Louis-Philippe  and Salakhutdinov, Ruslan",
    booktitle = "EMNLP-IJCNLP",
    year = "2019",
    url = "https://aclanthology.org/D19-1443",
    pages = "4344--4353",
}

@inproceedings{choromanski2021rethinking,
title={Rethinking Attention with {P}erformers},
author={Krzysztof Marcin Choromanski and Valerii Likhosherstov and David Dohan and Xingyou Song and Andreea Gane and Tamas Sarlos and Peter Hawkins and Jared Quincy Davis and Afroz Mohiuddin and Lukasz Kaiser and David Benjamin Belanger and Lucy J Colwell and Adrian Weller},
booktitle={ICLR},
year={2021},
url={https://openreview.net/forum?id=Ua6zuk0WRH}
}

@inproceedings{kitaev2020reformer,
title={Reformer: The Efficient {T}ransformers},
author={Nikita Kitaev and Lukasz Kaiser and Anselm Levskaya},
booktitle={ICLR},
year={2020},
url={https://openreview.net/forum?id=rkgNKkHtvB}
}

@article{wang2020linformer,
  title={Linformer: Self-attention with linear complexity},
  author={Wang, Sinong and Li, Belinda Z and Khabsa, Madian and Fang, Han and Ma, Hao},
  journal={arXiv preprint arXiv:2006.04768},
  year={2020},
  url={https://arxiv.org/abs/2006.04768},
}

@article{yang2024multi,
  title={Multi-candidate speculative decoding},
  author={Yang, Sen and Huang, Shujian and Dai, Xinyu and Chen, Jiajun},
  journal={arXiv preprint arXiv:2401.06706},
  year={2024},
  url={https://arxiv.org/abs/2401.06706}
}

@inproceedings{miao2023specinfer,
  title={{SpecInfer}: Accelerating generative large language model serving with tree-based speculative inference and verification},
  author={Miao, Xupeng and Oliaro, Gabriele and Zhang, Zhihao and Cheng, Xinhao and Wang, Zeyu and Zhang, Zhengxin and Wong, Rae Ying Yee and Zhu, Alan and Yang, Lijie and Shi, Xiaoxiang and Shi, Chunan and Chen, Zhuoming and Arfeen, Daiyaan and Abhyankar, Reyna and Jia, Zhihao},
  booktitle={ASPLOS},
  year={2024},
  url={https://dl.acm.org/doi/10.1145/3620666.3651335}
}

@inproceedings{liu2024kangaroo,
  title={Kangaroo: Lossless self-speculative decoding via double early exiting},
  author={Liu, Fangcheng and Tang, Yehui and Liu, Zhenhua and Ni, Yunsheng and Han, Kai and Wang, Yunhe},
  booktitle={NeurIPS},
  year={2024},
  url={https://proceedings.neurips.cc/paper_files/paper/2024/hash/16336d94a5ffca8de019087ab7fe403f-Abstract-Conference.html}
}

@inproceedings{cai2024medusa,
  title={Medusa: Simple {LLM} inference acceleration framework with multiple decoding heads},
  author={Cai, Tianle and Li, Yuhong and Geng, Zhengyang and Peng, Hongwu and Lee, Jason D and Chen, Deming and Dao, Tri},
  booktitle={ICML},
  year={2024},
  url={https://proceedings.mlr.press/v235/cai24b.html}
}

@inproceedings{qiu2024treebon,
  title={Tree{B}o{N}: Enhancing inference-time alignment with speculative tree-search and best-of-n sampling},
  author={Qiu, Jiahao and Lu, Yifu and Zeng, Yifan and Guo, Jiacheng and Geng, Jiayi and Wang, Huazheng and Huang, Kaixuan and Wu, Yue and Wang, Mengdi},
  booktitle={Findings of EMNLP},
  year={2025},
  url={https://aclanthology.org/2025.findings-emnlp.1140}
}

@inproceedings{brown2020language,
  title     = {Language Models are Few-Shot Learners},
  author    = {Tom B. Brown and Benjamin Mann and Nick Ryder and Melanie Subbiah and J. Kaplan and Prafulla Dhariwal and Arvind Neelakantan and Pranav Shyam and Girish Sastry and Amanda Askell and Sandhini Agarwal and Ariel Herbert-Voss and Gretchen Krueger and T. Henighan and R. Child and A. Ramesh and Daniel M. Ziegler and Jeff Wu and Clemens Winter and Christopher Hesse and Mark Chen and Eric Sigler and Mateusz Litwin and S. Gray and B. Chess and Jack Clark and Christopher Berner and Sam McCandlish and Alec Radford and I. Sutskever and Dario Amodei},
  booktitle   = {NeurIPS},
  year      = {2020},
  url = {https://proceedings.neurips.cc/paper/2020/hash/1457c0d6bfcb4967418bfb8ac142f64a-Abstract.html}
}

@misc{tran-thien2023optimal,
    title={An optimal lossy variant of speculative decoding},
    url={https://vivien000.github.io/blog/journal/a-provably-optimal-lossy-variant-of-speculative-decoding.html},
    howpublished={Unsupervised Thoughts (blog)},
    author={Tran-Thien, Vivien},
    year={2023}
}

@inproceedings{
narasimhan2025faster,
title={Faster Cascades via Speculative Decoding},
author={Harikrishna Narasimhan and Wittawat Jitkrittum and Ankit Singh Rawat and Seungyeon Kim and Neha Gupta and Aditya Krishna Menon and Sanjiv Kumar},
booktitle={ICLR},
year={2025},
url={https://openreview.net/forum?id=vo9t20wsmd}
}

@inproceedings{chen2024cascade,
  title={Cascade speculative drafting for even faster {LLM} inference},
  author={Chen, Ziyi and Yang, Xiaocong and Lin, Jiacheng and Sun, Chenkai and Chang, Kevin and Huang, Jie},
  booktitle={NeurIPS},
  pages={86226--86242},
  year={2024},
  url={https://proceedings.neurips.cc/paper_files/paper/2024/hash/9cb5b083ba4f5ca6bd05dd307a2fb354-Abstract-Conference.html}
}

@inproceedings{kim2023speculative,
  title={Speculative decoding with big little decoder},
  author={Kim, Sehoon and Mangalam, Karttikeya and Moon, Suhong and Malik, Jitendra and Mahoney, Michael W and Gholami, Amir and Keutzer, Kurt},
  booktitle={NeurIPS},
  pages={39236--39256},
  year={2023},
  url={https://proceedings.neurips.cc/paper_files/paper/2023/hash/7b97adeafa1c51cf65263459ca9d0d7c-Abstract-Conference.html}
}

\vfill
\pagebreak

\appendix

\section{Technical proofs}

\subsection{Proof of Observation~\ref{thm:generalization}}\label{apx:prf:generalization}

\generalization*

\begin{proof}
    Let $n$ be the selected token and $\vx$ be the context. Then at each step, the resulting distribution of the algorithm is:
    \begin{align}
        \Pr(n | \vx ) =& \Pr( n \sim p(\cdot|\vx) \text{ and } u \le \phi(x|\vx)) + \sum_{i=1}^{|V|} p(i | \vx) \Pr(n \sim g(\cdot|\vx) \text{ and } u > \phi(i|\vx)) \\
        =& p(n | \vx) \phi(n | \vx) + g(n | \vx) \E_{i\sim p(\cdot | \vx)}[1 - \phi(i | \vx)],
    \end{align}
    where the first term on the right hand side indicates the sampled token $n$ is accepted. The second term means the originally sampled token is rejected, and the current token $n$ comes from the recover distribution $g$. Since we would like $\Pr(n | \vx) = h(n | \vx)$, we have \begin{align}
        & p(n | \vx) \phi(n | \vx) + g(n | \vx) \E_{i\sim p(\cdot | \vx)}[1 - \phi(i | \vx)] = h(n | \vx) \label{eq:first-step}\\
    \iff & g(n | \vx) = \frac{h(n| \vx) - p(n|\vx)\phi(n|\vx)}{\E_{i\sim p(\cdot|\vx)}[1 - \phi(i | \vx)]}, \label{eq:g-expression}
    \end{align}
    hence proving the expression for $g$. Here, $\phi$ can be any function that maps to $[0,1]$ that makes $g$ a distribution. Since the expression of $g$ is self-normalizing, we only need to make sure that all $g(i | \vx)$ values are non-negative. Specifically, \begin{align}
         & 0\le g(i| \vx)  \\
         \impliedby & h(i| \vx) - p(i |\vx)\phi(i|\vx) \ge 0 \tag{$\text{Image}({\phi}) \subseteq [0, 1]$} \\
         \impliedby & \phi(i|\vx) \le \frac{h(i|\vx)}{p(i|\vx)}.
    \end{align}
    Again, with $\text{Image}({\phi}) \subseteq [0, 1]$, we have \begin{align}
        \phi(i|\vx) \le \min\left\{\frac{h(i|\vx)}{p(i|\vx)}, 1\right\},
    \end{align}
    which gives the optimal acceptance rate as $\min\left\{\tfrac{h(i|\vx)}{p(i|\vx)}, 1\right\}$.
\end{proof}

\subsection{Proof of Theorem~\ref{thm:general-solution}}\label{apx:prf:general-solution}

Before proceeding to the proof of the theorem, we first show the following technical lemma.
\begin{lemma}\label{lem:minimal-div}
Let $f: \mathbb{R}_+ \to \mathbb{R}$ be convex with $f(1) = 0$. For any $\alpha \in [0,1]$ and sub-distribution $\{q(i)\}_{i \in S}$ over $S$ with $Q := \sum_{i \in S} q(i) > 0$, the solution to:
\begin{align}
    \min_{\{h(i)\}} \quad & \sum_{i \in S} q(i) f\left(\frac{h(i)}{q(i)}\right) \label{eq:lem-obj} \\
    \text{s.t.} \quad & \sum_{i \in S} h(i) = \alpha, \quad h(i) \geq 0 \label{eq:lem-constraints}
\end{align}
is $h^*(i) = \frac{\alpha}{Q} q(i)$ for all $i \in S$.
\end{lemma}

\begin{proof}
Let $\lambda := \frac{\alpha}{Q}$. Define $\tilde{h}(i) := \lambda q(i)$. Then, we have
\begin{align}
    \sum_{i \in S} \tilde{h}(i) = \lambda Q = \alpha
\end{align}
satisfying the constraints. For any feasible $h \neq \tilde{h}$, define $r(i) := \frac{h(i)}{q(i)}$. By Jensen's inequality, we have
\begin{align}
    \frac{1}{Q}\sum_{i \in S} q(i) f(r(i)) \geq f\left(\frac{1}{Q}\sum_{i \in S} q(i) r(i)\right) = f\left(\frac{\alpha}{Q}\right)
\end{align}
with equality iff $r(i) = \lambda$ for all $i \in S$. Thus $\tilde{h}$ is the unique minimizer.
\end{proof}

We can now show the theorem below.
\generalsolution*
\begin{proof}
\begin{align}
    \max_{\vh} \ \ & \min\left\{\frac{h_n}{p(n | \vx_{<t})}, 1\right\} \\
    \mathrm{s.t.\ } & \vh \in \Delta^{|V|-1}, \\
    & D_f(\vh \| q(\cdot | \vx_{<t})) \le \delta.
\end{align}
Here, $\Delta^{|V|-1}$ denotes the probability simplex, and 
\begin{align}
    D_f(\vh \| q) = \sum_{i \in V} q(i) f\left(\frac{h(i)}{q(i)}\right)
\end{align}
is the $f$-divergence. The objective 
\begin{align}
    \min\left\{\frac{h_n}{p(n)}, 1\right\}
\end{align}
is maximized when $\frac{h_n}{p(n)}$ is as large as possible. However, since $\min\{ \cdot, 1 \}$ caps the value at 1, the maximum achievable is 1 (when $h_n \ge p(n)$). Thus, the problem reduces to maximizing $h_n$ under the constraints, as increasing $h_n$ directly improves the objective until $h_n \ge p(n)$. To maximize $h_n$, we allocate as much probability mass to $h_n$ as allowed by the constraints. Let $\gamma = h_n$. The remaining mass $1 - \gamma$ must be distributed over $i \neq n$. By Lemma~\ref{lem:minimal-div}, the optimal allocation for $i \neq n$ is:
\begin{align}
    h(i) = \frac{1 - \gamma}{1 - q(n)} q(i),
\end{align}
where $\frac{1 - \gamma}{1 - q(n)}$ ensures $\sum_{i \neq n} h(i) = 1 - \gamma$.

Substitute $h_n = \gamma$ and $h(i) = \frac{1 - \gamma}{1 - q(n)} q(i)$ into $D_f(\vh \| q)$:
\begin{align}
    D_f = q(n) f\left(\frac{\gamma}{q(n)}\right) + \sum_{i \neq n} q(i) f\left(\frac{1 - \gamma}{1 - q(n)}\right).
\end{align}
Simplify the second term using $\sum_{i \neq n} q(i) = 1 - q(n)$:
\begin{align}
    D_f = q(n) f\left(\frac{\gamma}{q(n)}\right) + (1 - q(n)) f\left(\frac{1 - \gamma}{1 - q(n)}\right).
\end{align}
The constraint $D_f \le \delta$ becomes an equality at optimality (since increasing $\gamma$ further would violate the constraint). Thus, $\gamma^*$ solves:
\begin{align}
    q(n) f\left(\frac{\gamma}{q(n)}\right) + (1 - q(n)) f\left(\frac{1 - \gamma}{1 - q(n)}\right) = \delta.
\end{align}

Finally, since $\gamma^*$ may exceed 1 (when $\delta$ is set too large to attain), it is truncated into $[q(n), 1]$ as a proper probability value.
\end{proof}

\subsection{Proof of Theorem~\ref{thm:hbound}}\label{apx:prf:hbound}

\hbound*

\begin{proof}
We work at a single step at $t$ and suppress the context $\vx_{<t}$. Fix $p$ and $q$ on a finite alphabet. For each drafted index $n$,
let $h_n$ be any target with $D_f(h_n\|q)\le \delta$. The conditional output is
\[
\vr_n=\phi_n(n) \ve_n+(1-\phi_n(n)) \vg(h_n),
\]
and the algorithm's one-step output is
\[
\vh_{\mathrm{alg}}=\sum_{n} p(n) \vr_n.
\]

Let
\begin{align}
\mathcal{H}_\delta
:=&\Bigl\{(h_n)_n:\ D_f(h_n\|q)\le \delta\ \ \forall n,\ \ q(i)=0\Rightarrow h_n(i)=0\Bigr\}, \\
F\bigl((h_n)_n\bigr):=&D_f\!\bigl(\vh_\text{alg}\|q\bigr).
\end{align}

Define
\begin{align}
\Gamma(\delta)\ :=\ \sup_{(h_n)\in \mathcal{H}_\delta}\ F\bigl((h_n)_n\bigr)
\ \in [0,\infty],
\end{align}
and note that $\Gamma$ depends only on $(p,q,f)$ and the budget $\delta$.
By construction,
\begin{align}
D_f\bigl(\vh_\text{alg}\|q\bigr)
\ \le\ \Gamma(\delta)\qquad\text{for every feasible family }(h_n)\in\mathcal{H}_\delta.
\end{align}
It is straightforward to show that $D_f\bigl(\vh_\text{alg}\|q\bigr) \le D_f(p \| q)$ given that the all-acceptance distribution is simply $p$.
Thus it remains to show that $\Gamma$ has the claimed shape:
non-decreasing, $\Gamma(0)=0$, and continuous in the extended-real sense.

\paragraph{Basic properties of $\Gamma$.}
\emph{(i) $\Gamma(0)=0$.} If $\delta=0$ then $h_n=q$ for all $n$, so $\vh_\text{alg}=q$ and thus $\Gamma(0)=D_f(q\|q)=0$.

\emph{(ii) Monotonicity.} If $\delta_2\ge \delta_1$ then
$\mathcal{H}_{\delta_1}\subseteq \mathcal{H}_{\delta_2}$, so
$\Gamma(\delta_1)\le \Gamma(\delta_2)$ by definition of the supremum.

\emph{(iii) Continuity.} We show right- and left-continuity. On a finite alphabet,
the set of probability distributions is compact (a simplex), and with
support alignment the feasible set $\mathcal{H}_\delta$ is closed (as the preimage 
of $[0,\delta]$ under the continuous function $\max_n D_f(\cdot\|q)$) and thus compact. 
The map $(h_n)_n\mapsto \vh_\text{alg}$ is continuous (operations involved 
are continuous on their domains), hence $F$ is continuous.

We first show the right-continuity. Let $\delta_k\downarrow\delta$. For each $k$ pick $(h_n^{(k)})_n\in\mathcal{H}_{\delta_k}$ 
with $F((h_n^{(k)})_n)\ge \Gamma(\delta_k)-\varepsilon_k$, where $\varepsilon_k\downarrow0$.
Since the alphabet is finite, the feasible families live in a finite product of simplices, 
which is compact; therefore, there exists a subsequence (not relabeled) such that 
$h_n^{(k)}\to h_n^\star$ for each $n$. By continuity of $D_f(\cdot\|q)$,
$D_f(h_n^\star\|q)=\lim_k D_f(h_n^{(k)}\|q)\le \lim_k \delta_k=\delta$,
so $(h_n^\star)_n\in \mathcal{H}_\delta$. Continuity of $F$ gives
\[
\limsup_{k\to\infty}\Gamma(\delta_k)
\ \le\ \lim_{k\to\infty}\bigl(F((h_n^{(k)})_n)+\varepsilon_k\bigr)
\ =\ F((h_n^\star)_n)\ \le\ \Gamma(\delta).
\]
Monotonicity gives $\Gamma(\delta)\le \liminf_{k\to\infty}\Gamma(\delta_k)$, hence
$\lim_{k\to\infty}\Gamma(\delta_k)=\Gamma(\delta)$.

We then show the left-continuity. Let $\delta_k\uparrow\delta$ and fix $\varepsilon>0$. Choose
$(h_n^\star)_n\in\mathcal{H}_\delta$ with $F((h_n^\star)_n)\ge \Gamma(\delta)-\varepsilon$.
For $\theta\in(0,1)$ define $h_{n,\theta}:=(1-\theta)h_n^\star+\theta q$.
By convexity of $D_f(\cdot\|q)$ in its first argument,
\[
D_f(h_{n,\theta}\|q)\le (1-\theta)D_f(h_n^\star\|q)+\theta D_f(q\|q)
\le (1-\theta)\delta < \delta,
\]
so $(h_{n,\theta})_n\in \mathcal{H}_{(1-\theta)\delta}$. By continuity of $F$, 
for sufficiently small $\theta>0$ we have
\[
F((h_{n,\theta})_n)\ge F((h_n^\star)_n)-\varepsilon \ge \Gamma(\delta)-2\varepsilon.
\]
For all large $k$ with $\delta_k > (1-\theta)\delta$, monotonicity gives
\[
\Gamma(\delta_k)\ \ge\ \Gamma((1-\theta)\delta)\ \ge\ F((h_{n,\theta})_n)
\ \ge\ \Gamma(\delta)-2\varepsilon.
\]
Thus $\liminf_{k\to\infty}\Gamma(\delta_k)\ge \Gamma(\delta)$, and since 
monotonicity gives $\limsup_{k\to\infty}\Gamma(\delta_k)\le \Gamma(\delta)$,
we have $\lim_{k\to\infty}\Gamma(\delta_k)=\Gamma(\delta)$.

In conclusion, by definition of $\Gamma$, for every feasible family $(h_n)\in\mathcal{H}_\delta$,
\[
D_f\!\bigl(\vh_\text{alg}\|q\bigr) \leq \Gamma(\delta),
\]
with $\Gamma$ non-decreasing, continuous on $[0,\infty)$,
and $\Gamma(0)=0$. This proves the theorem.
\end{proof}

\subsection{Proof of Proposition~\ref{thm:tas}}\label{apx:prf:tas}

\tas*
\begin{proof}
Given the optimization problem:
\begin{align}
    \max_{\vh} \quad & \min\left\{\frac{h_n}{p(n)}, 1\right\} \nonumber \\
    \text{s.t.} \quad & \vh \in \Delta^{|V|-1}, \label{eq:tas-simplex} \\
    & H(\vh,q ) \le H(q) + \delta, \label{eq:tas-ce}
\end{align}
where $H(q)$ is the entropy. The optimal solution concentrates mass on $\{n,m\}$. Equation \eqref{eq:tas-ce} is equivalent to
\begin{align}
    \sum_i [q(i) - h(i)]\log\frac{1}{q(i)} \ge -\delta. \label{eq:reformed-constraint}
\end{align}

To maximize $h_n = \gamma$, we must minimize the LHS of \eqref{eq:reformed-constraint}. Based on Lemma~\ref{lem:two-point}, the resulting distribution is always two-point distribution. Let $m = \argmax_i q(i)$. For fixed $h_n = \gamma$, the optimal allocation places all remaining mass on $m$:
\begin{align}
    h(i) = \begin{cases}
        \gamma, & i = n \\
        1-\gamma, & i = m \\
        0, & \text{otherwise}
    \end{cases}
\end{align}

Substituting the optimal form into \eqref{eq:tas-ce}:
\begin{align}
    \gamma\log\frac{1}{q(n)} + (1-\gamma)\log\frac{1}{q(m)} &\leq H(q) + \delta \nonumber \\
    \gamma\left(\log\frac{1}{q(n)} - \log\frac{1}{q(m)}\right) &\leq H(q) + \delta - \log\frac{1}{q(m)} \nonumber \\
    \gamma &\leq \frac{H(q) + \delta - \log\frac{1}{q(m)}}{\log\frac{q(m)}{q(n)}} \label{eq:gamma-bound}
\end{align}

Since $\gamma$ is a probability, its maximum is reached when

\begin{align}
\gamma = 1 \iff &
    \log\frac{q(m)}{q(n)} \leq H(q) + \delta - \log\frac{1}{q(m)} \nonumber \\
    \iff& q(n) \geq \exp\left(-H(q)\right)\exp(-\delta), \label{eq:main-result}
\end{align}
which is the acceptance rate used in TAS.

It should be noted that our theory here is used to reveal the soundness of the TAS acceptance function, without aiming to replicate the exact TAS algorithm. However, based on our framework, one can derive the exact TAS algorithm by adding an $H(\vh)=0$ constraint and an $\epsilon$ threshold to the cross-entropy limit, which we omitted for simplicity.
\end{proof}

In the proof above, we invoked the following technical lemma.
\begin{lemma}\label{lem:two-point}
For any $\gamma \in [0,1]$, the minimal value of $\sum_{i\neq n} [q(i) - h(i)]\log\frac{1}{q(i)}$ is achieved when:
\begin{align}
    h(m) = 1-\gamma,\quad h(i) = 0\ \forall i \neq n,m.
\end{align}
\end{lemma}
We provide the proof below.
\begin{proof}
Let $h(i) = \alpha_i(1-\gamma)$ for $i \neq n$, where $\sum_i \alpha_i = 1$. Then:
\begin{align}
    \sum_{i\neq n} [q(i) - \alpha_i(1-\gamma)]\log\frac{1}{q(i)}
\end{align}
is minimized when $\alpha_i$ concentrates on $m = \argmax q(i)$, since $\log\frac{1}{q(i)}$ is minimized at $i=m$.
\end{proof}

\subsection{Proof of Theorem~\ref{thm:solution}}\label{apx:prf:solution}

\solution*

\begin{proof}
We first compute the derivatives of $\Phi$ at $\gamma_0$:
\begin{align}
    \Phi(\gamma_0) =& \Phi'(\gamma_0) = 0, \\
    \text{and } \Phi''(\gamma_0) =& \frac{1}{q(n|\vx_{<t})(1 - q(n|\vx_{<t}))}.
\end{align}
The unique root in $[q(n|\vx_{<t}), +\infty)$ is then
\[
\gamma_0 + \sqrt{\frac{2\delta}{\Phi''(\gamma_0)}} = q(n|\vx_{<t}) + \sqrt{2\delta q(n|\vx_{<t})(1 - q(n|\vx_{<t}))}.
\]
We clip this value to the interval $[q(n|\vx_{<t}), 1]$ to ensure validity as a probability.
\end{proof}

\subsection{Proof of Corollary~\ref{thm:conservative}}\label{apx:prf:conservative}

\conservative*

\begin{proof}
Let $q :=  q\bigl(n\,\bigl|\,\vx_{<t}\bigr)$ for brevity and define the
quadratic‐approximate root
\begin{align}
\hat\gamma := q+\sqrt{2\delta\,q(1-q)} .
\end{align}
Because $\Phi'(\gamma)=\log\!\frac{\gamma}{q}-\log\!\frac{1-\gamma}{1-q}$,
we have $\Phi'(\gamma)>0$ for every $\gamma\in(q,1)$; hence
$\Phi$ is strictly increasing on $[q,1]$ and the equation
$\Phi(\gamma)=\delta$ admits a unique root $\gamma^\star\in(q,1]$.

Taylor's theorem with the Lagrange remainder, expanded at $\gamma_0=q$,
gives, for some $\xi\in(q,\gamma)$,
\begin{align}
\Phi(\gamma)
   = 
  \underbrace{\frac{\Phi''(q)}{2}\,(\gamma-q)^2}_{=: T_2(\gamma)}
   + 
  \frac{\Phi'''(\xi)}{6}\,(\gamma-q)^3 .
\end{align}
For the Bernoulli KL,
\begin{align}
\Phi''(\gamma)=\frac{1}{\gamma(1-\gamma)},\qquad
\Phi'''(\gamma)= -\frac{1-2\gamma}{\gamma^{2}(1-\gamma)^{2}}.
\end{align}
Whenever $\gamma\le\frac12$, the factor $1-2\gamma$ is non-negative and
therefore $\Phi'''(\xi)\le 0$.  It follows that
\begin{align}
\Phi(\gamma) \le T_2(\gamma)
         = \frac{(\gamma-q)^2}{2q(1-q)} ,
        \qquad \forall\gamma\in(q,\tfrac12] .
\tag{$\ast$}
\end{align}

Choose $\hat\gamma$ such that $T_2(\hat\gamma)=\delta$, this yields the expression given above.
If $\,\hat\gamma\le\frac12$ or equivalently
\begin{align}
\delta \le \frac{(1/2-q)^2}{2q(1-q)} \label{eq:equivalently}
\end{align}
then the above inequality gives $\Phi(\hat\gamma)<\delta$.  Since $\Phi$ is strictly
increasing, we obtain
\begin{align}
\hat\gamma < \gamma^\star .
\end{align}

This result ensures that our approximation never overestimates $\gamma$ when the verifier model is not confident about the current sampled token.
\end{proof}

\section{Additional experiments}

\paragraph{Mentored decoding.} A blog post proposed Mentored decoding~\citep{tran-thien2023optimal}, which uses binary search to generate a target distribution $\tilde q$ such that $D_\text{KL}(q \| \tilde q) \le \delta$. Compared with \name, there are two major differences: (1) Mentored decoding allows sampled tokens to be accepted even when the verifier has zero probability, violating the principle of adhering to the verifier's mode; (2) more importantly, the solution is found via a numerical optimization procedure, significantly slowing down the decoding speed and defeating the purpose of high-throughput decoding. We conduct additional experiments to compare \name and Mentored decoding (using $\delta=1$ as recommended).

As shown in Table~\ref{tab:additional}, Mentored decoding has the least acceptance rate gain at the cost of increasing the per-step generation time. For example, on GSM8K, the overall wall time is even longer than that of the naive SpS method by 20\%. In addition, the performance degrades severely on IFEval, echoing our analysis of the misplacement of the divergence arguments.

\paragraph{Speculative cascading.} More recently, \citet{narasimhan2025faster} proposed speculative cascading (SpecCas), which dynamically decides if the sampled token will be verified by the large model based on the difference between the two distributions. Essentially, it is mathematically equivalent to mixing the draft and verifier distributions as the target distribution at different steps. We therefore conduct experiments with SpecCas (the [OPT] variant and $\alpha=0.1$ for better quality).

The results in Table~\ref{tab:additional} show that SpecCas significantly increases the acceptance rate and the decoding speed. However, its generation quality is not as good as that of other methods, even when we choose hyper-parameters to favor higher generation quality. On the other hand, we also ran experiments with $\delta=10$ for \name. With a similar wall-time acceleration on GSM8K and GPQA, \name's generation quality is considerably higher. We hypothesize that this is due to the lack of explicit divergence control in SpecCas, whereas the other methods (especially \name) guarantee controlled ``distances.'' Given that the primary focus of this paper is to introduce a new, principled method, we leave a deeper investigation of these methods to future work.

\begin{table}[t]
  \caption{The results with Qwen~3~14B as verifier and Qwen~3~0.6B as drafter.}
  \label{tab:additional}
  \centering
    \resizebox{0.9\textwidth}{!}{%
  \begin{tabular}{llccccccccc}
    \toprule
    & & \multicolumn{3}{c}{GSM8K} & \multicolumn{3}{c}{IFEval} & \multicolumn{3}{c}{GPQA} \\
    \cmidrule(r){3-5}  \cmidrule(r){6-8} \cmidrule(r){9-11}
    $m$ & Name      &  Score$^\uparrow$ & AL$_{m}^\uparrow$  & Wall$^\downarrow$  &  Score$^\uparrow$ & AL$_{m}^\uparrow$  & Wall$^\downarrow$ &  Score$^\uparrow$ & AL$_{m}^\uparrow$  & Wall$^\downarrow$ \\
    \midrule 
    \multirow{6}{*}{10}& SpS & 91.12 & 4.27 & 1.00x & 85.03 & 2.19 & 1.00x & 39.39& 3.37 & 1.00x \\
    & Mentored & 91.66 & 4.51 & 1.20x & 61.37 & 2.88 & 0.96x & 40.91 & 4.31 & 0.93x \\
    & SpecCas & 88.40 & 6.42 &  0.85x & 69.50 & 5.02 & 0.54x & 32.83 & 6.27 & 0.68x \\
    & TAS  & 92.65 & 5.24 & 0.86x &  86.14 & 3.00 & 0.82x & 38.89 & 4.99 & 0.72x \\
    & \name 1 & 93.10 & 5.44 & 0.87x & 85.96 & 3.03 & 0.78x & 43.43 & 5.16 & 0.70x \\
    & \name 10 & 92.72 & 5.73 & 0.83x & 84.66 & 3.41 & 0.74x & 39.40 & 5.71 & 0.69x \\
    \bottomrule
  \end{tabular}}
\end{table}

\paragraph{Evaluations on Spec-Bench.}
To provide a more comprehensive assessment of \name across diverse scenarios, we conduct evaluations on Spec-Bench~\citep{xia2024unlocking}, a unified benchmark designed to test speculative decoding methods across multiple distinct domains, including multi-turn conversation (MT-Bench), translation (WMT), summarization (CNN/DM), question answering (natural questions), mathematical reasoning (GSM8K), and retrieval-augmented generation (RAG). This broad coverage ensures that the observed speedups are not limited to specific task types but are consistent across varied real-world applications. We use the Qwen~3~14B model as the verifier and the 0.6B model as the drafter, maintaining a temperature of 0.6.

\begin{table}[t]
\caption{Speedup comparison on Spec-Bench using Qwen~3~14B as the verifier and Qwen~3~0.6B as the drafter. We report the speedup ratio relative to standard autoregressive decoding. ``Accepted'' denotes the mean number of accepted tokens per step.}\label{tab:spec_bench}
\centering
\resizebox{0.9\textwidth}{!}{%
\begin{tabular}{lcccccccc}
\toprule 
& MT Bench & Trans. & Summ. & QA & Math & RAG & AL$_{10}$ & Overall \\
\midrule
SpS & 2.01$\times$ & 1.40$\times$ & 1.92$\times$ & 1.85$\times$ & 1.83$\times$ & 1.86$\times$ & 3.20 & 1.81$\times$ \\
\name ($\delta=1$) & \textbf{2.09$\times$} & \textbf{1.40$\times$} & \textbf{2.04$\times$} & \textbf{1.95$\times$} & \textbf{1.86$\times$} & \textbf{1.92$\times$} & \textbf{3.29} & \textbf{1.88$\times$} \\
\bottomrule
\end{tabular}}
\end{table}

The results are summarized in Table~\ref{tab:spec_bench}. \name is tested without any hyper-parameter tuning ($\delta=1$). However, it immediately yields acceleration over the SpS baseline. In addition, \name consistently outperforms SpS across different domains, achieving an overall speedup of $1.88\times$ (+88\% gain over autoregressive decoding). This significant reduction in compute cost is achieved without additional training. It is worth noting that these speeds are measured using the HuggingFace Transformers framework~\citep{wolf2019huggingface0s}, which is less optimized for speculative sampling. We anticipate that the real-world performance gains would be even larger with a better implementation such as vLLM~\citep{kwon2023efficient}, as indicated by our other experiments.

\paragraph{Impact of draft model size.}
We employ same-family models to ensure aligned tokenization, consistent with standard practice~\citep{leviathan2023fast, chen2023accelerating}. To investigate the impact of drafter capacity, we evaluate \name on GSM8K using a Qwen~3~14B verifier with varying drafter sizes (Table~\ref{tab:draft_size}).

\begin{table}[t]
\caption{Ablation on GSM8K using Qwen~3~14B verifier with different drafter sizes ($\delta=1$).}\label{tab:draft_size}
\centering
\resizebox{0.4\textwidth}{!}{%
\begin{tabular}{lccc}
\toprule
Draft Size & Score & AL & Rej \\
\midrule
Verifier (Oracle) & 91.71 & - & - \\
0.6B & \textbf{93.10} & 5.44 & -32\% \\
1.7B & 92.50 & 6.78 & -60\% \\
4B & 92.57 & \textbf{7.76} & \textbf{-76\%} \\
\bottomrule
\end{tabular}}
\end{table}

Increasing the drafter size to 4B significantly boosts the mean accepted length (AL) to 7.76 and reduces rejection rates by 76\%, while maintaining high task accuracy. These results confirm that \name effectively scales with stronger drafters, translating improved draft quality into greater decoding efficiency.

\section{Case study}

In this section, we discuss whether the choice of $\delta$ affects qualitative measures such as reasoning ability. We gather the results of the first sample from GSM8K, where $\delta$ is set to different values when running \name with the Qwen~3~0.6B + 14B model pair.

From the case study in Table~\ref{tab:case-study}, we can see that the reasoning is poor and lengthy when $\delta$ is large (more divergence allowed). Consequently, the result is wrong due to the low-quality chain-of-thought. This confirms that the divergence control in \name helps maintain qualitative measures.

\begin{table}[htb]
\centering
\caption{Generations for different values of $\delta$.}
\label{tab:case-study}
\begin{tabular}{lp{0.95\textwidth}}
\toprule
Task & \ttfamily\detokenize{Question: Janet's ducks lay 16 eggs per day. She eats three for breakfast every morning and bakes muffins for her friends every day with four. She sells the remainder at the farmers' market daily for \$2 per fresh duck egg. How much in dollars does she make every day at the farmers' market?} \\
\midrule
$\delta=1$ & \ttfamily\detokenize{<think>\nOkay, let me try to figure out how much Janet makes each day at the farmers' market. Hmm, so the problem says she has ducks that lay 16 eggs per day. Then she eats three for breakfast and bakes muffins with four eggs each day. Then she sells the remaining eggs for $2 each. Alright, let's break this down step by step.\n\nFirst, total eggs per day: 16 eggs. Then she uses some for breakfast and baking. Let's see, for breakfast, she eats 3 eggs. Then for baking, she uses 4 eggs. So total eggs used each day would be 3 + 4 = 7 eggs. Therefore, the remaining eggs would be 16 - 7 = 9 eggs. Then she sells those 9 eggs at $2 each. So, how much does she make? 9 eggs multiplied by $2, which is 18 dollars.\n\nWait, let me check again. If she uses 3 eggs for breakfast and 4 for baking, that's 7. 16 minus 7 is indeed 9. Then 9 times 2 is 18. Yeah, that seems right. So her daily earnings would be $18. Hmm, is there any possibility I made a mistake here? Maybe I should re-read the problem to make sure I didn't miss anything.\n\nThe problem says she lays 16 eggs per day. Then she eats 3 for breakfast and bakes muffins each day with four. So the total used is 3+4=7. Subtract that from 16, which leaves 9. Multiply by $2 for selling, 9*2=18. Yep, that's correct. I think that's the answer.\n</think>\n\nJanet's ducks lay 16 eggs per day. She uses 3 eggs for breakfast and 4 eggs for baking muffins, totaling 7 eggs used. The remaining eggs are $16 - 7 = 9\$. She sells these at $2 per egg, earning $9 \\times 2 = 18\$.  \n\n#### 18} \\
\midrule
$\delta=10^6$ & \ttfamily\detokenize{<think>\nOkay, so let\'s try to figure out how much Janet makes every day at the farmers\' market. Let me start by breaking down the problem step by step.} \textbf{[lengthy CoT omitted]} \detokenize{so total eggs used is 4 * 4 = 16. Then, after eating 3 for breakfast, leftover is 16 - 3 = 13, then she sells them at 2 each: 13 * 2 = $26. \n\nYes, that seems to make sense. Therefore, the answer should be 26. But let me check if the question is asking "how much money does she make every day at the farmers\' market?" So yes, selling the remainder at $2 per egg. Therefore, answer is $26. \n\nTherefore, the answer is 26, which in the previous problems they had 5. This is a different case. So I think that\'s correct.\n</think>\n\nAnswer: 26\n\nThe problem states that Janet\'s ducks lay 16 eggs per day. She eats 3 eggs for breakfast and 4 eggs to make muffins daily. After these expenses, the remaining eggs are $ 16 - 3 - 4 = 9 $, which are sold at $2 per egg. Therefore, the total money she earns at the farmers\' market is $ 9 \\times 2 = 18 $. \n\nHowever, the phrase "bakes muffins for her friends every day with four" might imply that she uses 4 eggs per muffin, meaning she makes 4 muffins. If she uses 4 eggs each muffin, the total eggs consumed would be $ 4 \\times 4 = 16 $. Since she eats 3, the remaining eggs are $ 16 - 3 = 13 $, which are sold for $ 13 \\times 2 = 26 $. Hence, the correct answer is **26**. \n\n#### 26} \\
\bottomrule
\end{tabular}
\end{table}

\section{Broader impact and future directions}

\paragraph{Broader impact.} By improving the inference efficiency of large language models without sacrificing output quality, our method reduces computational costs and energy consumption. This contributes to more sustainable AI deployment, broadens access to high-performance language models, and supports environmentally conscious machine learning practices. Additionally, Cactus can enable faster, lower-cost applications in education, healthcare, and low-resource settings.

\paragraph{Future directions.} Our goal in this paper is to introduce and analyze the draft-and-verify framework, not to exhaustively optimize every dimension of the system. Accordingly, we identify several extensions and leave them for future exploration by the community: (1) \emph{Model scale.} We capped evaluation at 32B parameters to keep the methodology clear and costs tractable. Pushing to substantially larger backbones could reveal scaling behavior (e.g., effects on acceptance rates, latency, and robustness) and is best investigated in follow-on work, including studies of scaling laws and distributed inference. (2) \emph{Model training.} We emphasize a training-free method to highlight the mechanism itself. While targeted tuning (e.g., LoRA for the draft~\citep{hu2021lora,wu2025ulpt}, verifier calibration, joint sequence-level distillation~\citep{wen2023f}) may further improve proposal quality and reduce disagreement error, such engineering is orthogonal to our core contribution and thus deferred. (3) \emph{Memory usage.} Draft-and-verify introduces extra memory for the draft model and caches. Techniques like quantization, weight sharing, cache reuse, selective offloading, and early-exit heuristics could lower this footprint, but a thorough treatment would distract from the main result; we leave these optimizations to future work. (4) \emph{Leveraging ensemble effects.} In our main experiments, we observe that \name often performs better than the verifier model. For example, Cactus surpasses the verifier's accuracy by 2 standard deviations on both IFEval and GPQA. We hypothesize that this is because \name enables a ``healthy'' ensemble effect by combining two model distributions. Leveraging ensemble effects in speculative sampling could be explored in future work.

\section{The use of large language models}

Throughout this paper (with this paragraph being an exception), we use large language models to help identify grammar errors. Specifically, we prompt ChatGPT to ``Revise grammar errors with minimal changes of the original text'', followed by the latex source code of each paragraph. In addition, we use ChatGPT and DeepSeek~R1 to triple-check all technical proofs. The code for plotting all the figures is initially generated by ChatGPT, which is further revised by the authors according to the authors' aesthetics. We certify that the originality and scientific contributions of our method do not come from any large language models.

\end{document}